\documentclass{article}
\PassOptionsToPackage{numbers, compress}{natbib}

\input{preamble.tex}
\graphicspath{{figs/}}
\usepackage[preprint]{neurips_2019}

\usepackage[utf8]{inputenc} 
\usepackage[T1]{fontenc}    
\usepackage{hyperref}       
\usepackage{url}            
\usepackage{booktabs}       
\usepackage{amsfonts}       
\usepackage{nicefrac}       
\usepackage{microtype}      
\usepackage{wrapfig}
\usepackage{arydshln}
\usepackage{wrapfig}
\usepackage{natbib}
\newcommand{\graphbloom}{{\it GraphBloom}\xspace}
\newcommand{\rpm}{\sbox0{$1$}\sbox2{$\scriptstyle\pm$}
  \raise\dimexpr(\ht0-\ht2)/2\relax\box2 }

\title{Graph DNA: Deep Neighborhood Aware \\
Graph Encoding for Collaborative Filtering}

\author{%
  Liwei Wu \\
  Department of Statistics\\
  University of California, Davis\\
  Davis, CA 95616 \\
  \texttt{liwu@ucdavis.edu} \\
  \And
  Hsiang-Fu Yu\\
  Amazon\\
  Palo Alto, CA 94301 \\
  \texttt{rofu.yu@gmail.com} \\
  \And
  Nikhil Rao\\
  Amazon\\
  Palo Alto, CA 94301 \\
  \texttt{nikhilrao86@gmail.com} \\
  \And
  James Sharpnack \\
  Department of Statistics\\
  University of California, Davis\\
  Davis, CA 95616 \\
  \texttt{jsharpna@ucdavis.edu} \\
  \And
  Cho-Jui Hsieh \\ 
  Department of Computer Science\\
  University of California, Los Angles\\
  Los Angles, CA 90095 \\
  \texttt{chohsieh@cs.ucla.edu}\\
}
\begin{document}

\maketitle

\begin{abstract}

In this paper, we consider recommender systems with side information in the form of graphs. Existing collaborative filtering algorithms mainly utilize only immediate neighborhood information and do not efficiently take advantage of deeper neighborhoods beyond 1-2 hops. The main issue with exploiting deeper graph information is the rapidly growing time and space complexity when incorporating information from these neighborhoods. In this paper, we propose using Graph DNA, a novel Deep Neighborhood Aware graph encoding algorithm, for exploiting multi-hop neighborhood information. DNA encoding computes approximate deep neighborhood information in linear time using Bloom filters, 
and results in a per-node encoding whose dimension is logarithmic in the number of nodes in the graph. It can be used in conjunction with both feature-based and graph-regularization-based collaborative filtering algorithms.  Graph DNA has the advantages of being memory and time efficient and providing additional regularization when compared to directly using higher order graph information. We provide theoretical performance bounds for graph DNA encoding, and experimentally show that graph DNA can be used with 4 popular collaborative filtering algorithms, leading to a performance boost with little computational and memory overhead. 
\end{abstract}


\section{Introduction}

Recommendation systems are increasingly prevalent due to content delivery platforms, e-commerce websites, and mobile apps \cite{shani2008mining}.
Classical collaborative filtering algorithms use matrix factorization to identify latent features that describe the user preferences and item meta-topics from partially observed ratings \cite{koren2009matrix}.
In addition to rating information, many real-world recommendation datasets also have a wealth of side information in the form of graphs, and incorporating this information often leads to performance gains.
For example, \cite{rao2015collaborative, zhou2012kernelized} propose to add a graph regularization to the matrix factorization formulation to exploit additional graph structure; and \cite{liang2016factorization} conduct a co-factorization of the graph and rating matrix.
However, each of these only utilizes the immediate neighborhood information of each node in the side information graph.
 More recently, \cite{berg2017graph} incorporated graph information when learning features with a Graph Convolution Network (GCN) based recommendation algorithm.
GCNs \cite{kipf2016semi} constitute flexible methods for incorporating graph structure beyond first-order neighborhoods, but their training complexity typically scales rapidly with the depth, even with sub-sampling techniques \cite{chen2018stochastic}.
Intuitively, exploiting higher-order neighborhood information could benefit the generalization performance, especially when the graph is sparse, which is usually the case in practice. The main caveat of exploiting higher-order graph information is the high computational and memory cost when computing higher-order neighbors since the number of $t$-hop neighbors typically grows exponentially with $t$. 

In this paper, we aim to utilize higher order graph information without introducing much computational and memory overhead. To achieve this goal, we propose 
a Graph Deep Neighborhood Aware (Graph DNA) encoding, which approximately captures the higher-order neighborhood information of each node via Bloom filters \cite{bloom1970space}.
Bloom filters encode neighborhood sets as $c$ dimensional 0/1 vectors, where $c = O(\log n)$ for a graph with $n$ nodes, which approximately preserves membership information.
This encoding can then be combined with both graph regularized or feature based collaborative filtering algorithms, with little computational and memory overhead. 
In addition to computational speedups, we find that Graph DNA achieves better performance over competitors, which we hypothesize is due to the unique nature of Graph DNA and its connection to the shortest path length distance. We make this connection precise with theoretical bounds in Section \ref{sec:graph-bloom}.


We show that our Graph DNA encoding can be used with several collaborative filtering algorithms: graph-regularized matrix factorization with explicit and implicit feedback \cite{zhou2012kernelized, rao2015collaborative}, co-factoring \cite{liang2016factorization}, and GCN-based recommendation systems \cite{monti2017geometric}. 
In some cases, using information from deeper neighborhoods (like $4^{th}$ order) yields a 15x increase in performance, with graph DNA encoding yielding a 6x speedup compared to directly using the $4^{th}$ power of the graph adjacency matrix. 

\paragraph{Related Work}
\label{sec:bloom-filter}
\vspace{-2mm}
Matrix factorization has been used extensively in recommendation systems with both explicit \cite{koren2009matrix} and implicit \cite{hu2008collaborative} feedback. Such methods compute low dimensional user and item representations; their inner product approximates the observed (or to be predicted) entry in the target matrix. To incorporate graph side information in these systems, \cite{rao2015collaborative, zhou2012kernelized} used a graph Laplacian based regularization framework that forces a pair of node representations to be similar if they are connected via an edge in the graph. In \cite{yu2017unified}, this was extended to the implicit feedback setting. \cite{liang2016factorization} proposed a method that incorporates first-order information of the rating bipartite graph into the model by considering item co-occurrences. More recently, GC-MC \cite{berg2017graph} used a GCN approach performing convolutions on the main bipartite graph by treating the first-order side graph information as features, and \cite{monti2017geometric} proposed combining GCNs and RNNs for the same task. 

Methods that use higher order graph information are typically based on taking random walks on the graphs \cite{gori2007itemrank}. 
\cite{jamali2009trustwalker}  extended this method to include graph side information in the model. Finally, the PageRank \cite{page1999pagerank} algorithm can be seen as computing the steady state distribution of a Markov network, and similar methods for recommender systems was proposed in \cite{abbassi2007recommender, xie2015edge}. 

For a complete list of related works of representation learning on graphs, we refer the interested user to \cite{hamilton2017representation}. For the collaborative filtering  setting, \cite{monti2017geometric, berg2017graph} use Graph Convolutional Neural Networks \cite{defferrard2016convolutional}, but with some modifications. Standard GCN methods without substantial modifications cannot be directly applied to collaborative filtering rating datasets, including well-known approaches like GCN \cite{kipf2016semi} and GraphSage \cite{hamilton2017inductive}, because they are intended to solve semi-supervised classification problem over graphs with nodes' features. 
PinSage \cite{ying2018graph} is the GraphSage extension to non-personalized graph-based recommendation algorithm but not meant for collaborative filtering problems. GC-MC \cite{berg2017graph} extend GCN to collaborative filtering, albeit less scalable than \cite{ying2018graph}. Our Graph DNA scheme can be used to obtain graph features in these extensions.
In contrast to the above-mentioned methods involving GCNs, we do not use any loss function to train our graph encoder. This property makes our graph DNA suitable for both transductive as well as inductive problems.

Bloom filters have been used in Machine Learning for multi-label classification \cite{cisse2013robust}, and for hashing deep neural network models representations \cite{shi2009hash, han2015deep, courbariaux2015binaryconnect}. However, to the best of our knowledge, they have not been used to encode graphs, nor has this encoding been applied to recommender systems


\section{Methodology}
\label{sec:dna-algo}
We consider the problem of recommender system with a partially observed rating matrix $R$ and a 
Graph that encodes side information $G$. In this section, we will introduce the Graph DNA algorithm for encoding deep neighborhood information in $G$. In the next section, we will show how this encoded information can be applied to various graph based recommender systems. 

\subsection{Bloom Filter}
The Bloom filter~\cite{bloom1970space} is a probabilistic data structure designed
to represent a set of elements. Thanks to its space-efficiency and
simplicity, Bloom filters are applied in many real-world applications such as
database systems \cite{borthakur2011apache, chang2008bigtable}.
A Bloom filter $\cB$
consists of $k$ independent hash functions
$h_t(x) \rightarrow \cbr{1, \ldots, c}$. 
The Bloom filter $\cB$ of size $c$ can be represented as a length $c$ bit-array $\bb$. More details about Bloom filters can be found in~\cite{broder2004network}.
Here we highlight a few desirable properties of Bloom filters essential to
our graph DNA encoding: 
\begin{enumerate}
  \item Space efficiency: classic Bloom filters use $1.44 \log_2 (1/\epsilon)$
    of space per inserted key, where $\epsilon $ is the false positive rate
    associated with this Bloom filter.
  \item Support for the union operation of two Bloom filters: the Bloom filter
    for the union of two sets can be obtained by performing bitwise `OR' operations
    on the underlying bit-arrays of the two Bloom filters.
  \item Size of the Bloom filter can be approximated by the number of
    nonzeros in the underlying bit array: in particular, given a Bloom filter
    representation $\cB(A)$ of a set $A$: the number of elements of $A$ can be
    estimated as
        $\abs{A} \approx - \frac{c}{k} \log\rbr{ 1 - \frac{\nnz(\bb)}{c}}$,
    where $\nnz(\bb)$ is the number of non-zero elements in array $\bb$.
    As a result, the number of common nonzero bits of $\cB(A_1)$ and $\cB(A_2)$
    can be used as a proxy for $\abs{A_1 \cap A_2}$.
\end{enumerate}

\begin{algorithm}[H]
  \caption{Graph DNA Encoding with Bloom Filters}
\label{alg:graph-bloom}
\begin{algorithmic}[1]
\Require $G$: a graph of $n$ nodes, $c$: the length of codes, $k$: the number of hash functions, $d$: the number of iterations, $\theta$: tuning parameter to control the number of elements hashed.
\Ensure $B \in \cbr{0,1}^{n \times c}$: a boolean matrix to denote the bipartite relationship between $n$ nodes and $c$ bits.
  \begin{compactitem}
    \item $\cH \leftarrow \cbr{\mathtt{h_t}(\cdot): t = 1,\ldots, k}$ \Comment{Pick $k$ hash functions}
    \item ${\bf for}$ $i = 1,\ldots,n$: \Comment{\graphbloom Initialization}
      \begin{compactitem}
        \item $\cB^{0}\mathtt{[i]} \leftarrow \mathtt{BloomFilter}(c, \cH)$
        \item $\cB^{0}\mathtt{[i].add}(i)$
      \end{compactitem}
    \item ${\bf for}$ $s = 1,\ldots,d$: \Comment{$d$ times neighborhood propagations }
      \begin{compactitem}
      \item ${\bf for}$ $i = 1,\ldots,n$:
          \begin{compactitem}
            \item ${\bf for } $ $j \in \cN_1(i)$: \Comment{degree-1 neighbors}
              \begin{compactitem}
                \item ${\bf if}$ $|\cB^{s}\mathtt{[i]}| > \theta$: break;
                \item $\cB^{s}\mathtt{[i].union}(\cB^{s-1}\mathtt{[j]})$
              \end{compactitem}
          \end{compactitem}
      \end{compactitem}
    \item $B_{ij}\leftarrow \cB^{d}\mathtt{[i].b[j]}\ \forall (i,j) \in [n] \times [c]$
  \end{compactitem}
\end{algorithmic}
\end{algorithm}

\newcommand{\BF}[1]{\cB\mathtt{[#1]}}
\subsection{Graph DNA Encoding Via Bloom Filters}
\label{sec:graph-bloom}

Now we introduce our Graph DNA encoding. The main idea is to encode the deep (multi-hop) neighborhood aware embedding for each node in the graph
approximately using the Bloom filter, which helps avoid performing computationally expensive graph adjacency matrix multiplications.
In Graph DNA, we have Bloom filters $\BF{i}, i=1,...,n$ for the $n$ graph nodes. All the Bloom filters $\BF{i}$ share the same $k$ hash functions. The role of $\BF{i}$ is to store the deep neighborhood information of the $i$-th node.  Taking advantage of the union operations of Bloom filters, one node's neighborhood information can be propagated to its neighbors in an iterative manner using gossip algorithms \cite{shah2009gossip}. Initially, each $\BF{i}$ contains only the node itself. At the $s$-th iteration, $\BF{i}$ is updated by taking union with node $i$'s immediate neighbors' Bloom filters $\BF{j}$. By induction, we see that after the $d$ iterations, $\BF{i}$ represents $\cN_d(i) :=  \cbr{j: \text{distance}_{G}(i, j) \le d}$,
where $\text{distance}_{G}(i, j)$ is the shortest path distance between nodes $i$ and $j$ in $G$.
As the last step, we stack array representations of all Bloom filters and form a sparse matrix $B \in \cbr{0, 1}^{n \times c}$, where 
the $i$-th row of $B$ is the bit representation of $\BF{i}$. As a practical measure, to prevent over-saturation of Bloom filters for popular nodes in the graph, we add a hyper-parameter $\theta$ to control the max saturation level allowed for Bloom filters. This would also prevent hub nodes dominating in graph DNA encoding. The pseudo-code for the proposed encoding algorithm is given in Algorithm~\ref{alg:graph-bloom}. 
We use graph DNA-$d$ to denote our obtained graph encoding after applying Algorithm~\ref{alg:graph-bloom} with $s$ looping from 1 to $d$. We also give a simple example to illustrate how the graph DNA is encoded into Bloom filter representations in Figure~\ref{fig:graph-bloom}. Our usage of Bloom filters is very different from previous works in \cite{pozo2016item, serra2017getting, shinde2016user}, which use Bloom filter for standard hashing and is unrelated to graph encoding. 

It is intuitive that the number of 1-bits in common between two Bloom filters should be closely related to the size of the intersection of their neighborhoods. However, there may also be false positives in the bit-representations. 
We control precisely the size of such false positives and the number of common bits in the following theorem. 
The following theorem only applies to Bloom filters without the max saturation threshold $\theta$.

\begin{theorem}
\label{thm:bf}
Suppose that the Bloom filters have $c$ bits and the $k$ hash functions are independent for all nodes.
Consider two nodes $i,j = 1,\ldots,n$, their $d$-hop neighborhoods $\cN_d(i),\cN_d(j)$, and their $d$-depth Bloom filters $\BF{i},\BF{j}$, respectively.
Let $Q_{i,j}$ be the number of common 1-bits in the Bloom filters of $i,j$ (the inner product of the vectorized Bloom filters, $\langle\BF{i}, \BF{j}\rangle$).
There exists universal constants $C_0, C_1$, such that for any $\gamma > 0$, with probability $1 - \gamma$,
\begin{equation}
\label{eq:Q_upper}
Q_{i,j} \le \left( 1 + \frac{1}{C_0} \ln \frac{C_1}{\gamma} \right) \cdot \left( k^2 \frac{|\cN_d(i) \triangle \cN_d(j)|^2}{4c} + \frac{c k |\cN_d(i)\cap\cN_d(j)|}{c-1} \right),
\end{equation}
where $\cN_d(i) \triangle \cN_d(j)$ denotes the symmetric difference.
Furthermore, for any $\delta \in (0,1)$ there exists a constant $\alpha > 0$ such that if $c \alpha > k |\cN_d(i)\cap\cN_d(j)|$ then 
\begin{equation}
\label{eq:Q_lower}
\mathbb P \left\{ Q_{i,j} > (1 - \delta) k |\cN_d(i)\cap\cN_d(j)| \right\} \ge 1-  e^{- \frac 13 (1 - \delta) \delta^2 k |\cN_d(i)\cap\cN_d(j)|}.
\end{equation}
\end{theorem}

This theorem is a corollary of the more precise Theorem \ref{thm:main}, which is stated in the Appendix.
In order to establish these results, we provide Lemma \ref{lem:NA}, which demonstrates that the bits of Bloom filters are negatively associated (basic properties of negative associativity can be found in \cite{dubhashi1998balls,joag1983negative}), and this property is preserved under bitwise `or' and `and' operations on independent Bloom filters.
As a result, $Q_{i,j}$ enjoys Chernov-Hoeffding bounds, and the result follows by analyzing its expectation.

\begin{remark}
When the neighborhoods have no intersection, $|\cN_d(i)\cap\cN_d(j)| = 0$ then we have that $Q_{i,j} = O_P ( k^2 |\cN_d(i)\cup\cN_d(j)|^2 / c )$ which is approaching $0$ when $k |\cN_d(i)\cup\cN_d(j)|  = o( \sqrt c)$ (the number of bits in the Bloom filters are taken to be large enough) by \eqref{eq:Q_upper}. 
\end{remark}
\begin{remark}
Generally, \eqref{eq:Q_lower} states that when the number of hashed functions for the intersection is large, $k |\cN_d(i)\cap\cN_d(j)| \rightarrow \infty$, but dominated by the number of bits, $k |\cN_d(i)\cup\cN_d(j)|  = o(c)$, then we have that $\lim \left( Q_{i,j} / (k |\cN_d(i)\cap\cN_d(j)|) \right) \ge 1$ almost surely.  
For fixed neighborhood sizes, we can take $c \propto \log n$ and $k \propto \log \log n$, and obtain that $Q_{i,j}/k = O_P(|\cN_d(i)\cap\cN_d(j)|)$ by \eqref{eq:Q_upper} and $Q_{i,j}/k = \Omega_P(|\cN_d(i)\cap\cN_d(j)|)$ by \eqref{eq:Q_lower}.
\end{remark}

Graph DNA encodes deep neighborhood information such that for any two nodes whose shortest path length distance is at most $2d$, we only need to run Algorithm~\ref{alg:graph-bloom} for $d$ iterations. For example, in Figure~\ref{fig:DNA}, nodes $x$ and $y$ are 6 hops away on the shortest path, but they will start to share their bits' representations after 3 iterations because the node $z$'s information can be propagated to node $x$ and $y$ after exactly 3 iterations.
Theorem \ref{thm:bf} and the remarks that follow it demonstrate that by increasing the number of hash functions and the number of bits in the Bloom filter, the number of common 1-bits in these Bloom filters becomes an accurate surrogate for $|\cN_d(x)\cap\cN_d(y)|$.

The $n \times c$ Bloom filter matrix $B$ can also be viewed as the adjacency matrix of a bipartite graph between the $n$ nodes in the original graph and $c$ meta nodes of Bloom filters. 
In this way, nodes $x$ and $y$ have a bit in common in their Bloom filter representations if they are both connected to at least one meta node in $B$. This property saves memory and time required for graph encoding, allowing us to use $B$ instead of the adjacency matrix $G$ in graph Laplacian regularization methods \cite{rao2015collaborative}, and to use $B$ as side features in graph convolutional network based geometric matrix factorization algorithm \cite{monti2017geometric, berg2017graph} with little computational and memory overhead. We elaborate on this in the following section.

\begin{figure*}[ht]
\begin{center}
\centerline{\includegraphics[width=0.8\columnwidth]{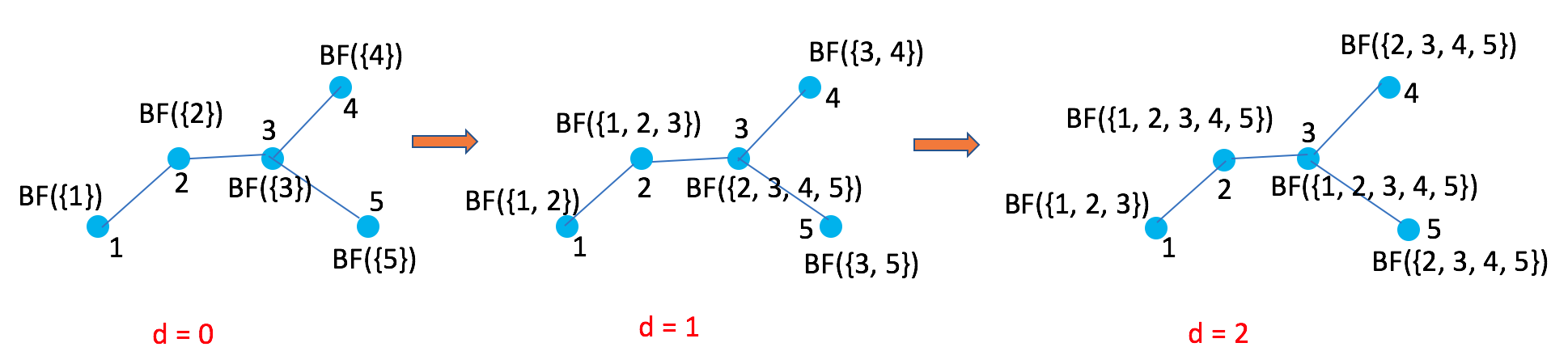}}
\end{center}
\caption{Illustration of Algorithm~\ref{alg:graph-bloom}: the graph DNA encoding procedure. The curly brackets at each node indicate the nodes encoded at a particular step.  At $d=0$ each node's Bloom filter only encodes itself, and multi-hop neighbors are included as d increases.}
\label{fig:graph-bloom}
\vspace{-10pt}
\end{figure*}

\begin{figure}[ht]
\begin{center}
\centerline{\includegraphics[width=0.9\columnwidth]{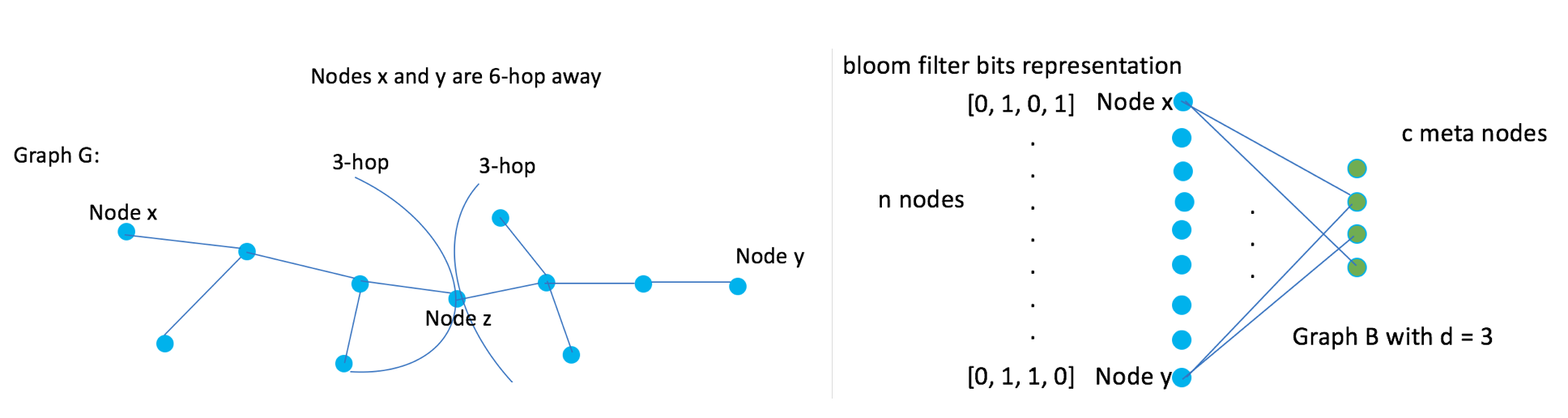}}
\end{center}
\caption{Illustration of our proposed DNA encoding method (DNA-3), with the corresponding bipartite graph representation.}
\label{fig:DNA}
\vspace{-10pt}\end{figure}

\section{Collaborative Filtering with Graph DNA}
\label{sec:app}
Suppose we are given the sparse rating matrix $R \in \dR^{n \times m}$  with $n$ users and $m$ items, and a graph $G \in \dR^{n\times n}$ encoding relationships between users. For simplicity, we do not assume a graph on the $m$ items, though including it is straightforward. 
%

\subsection{Graph Regularized Matrix Factorization}\label{sec:grmf}


The objective function of Graph Regularized Matrix Factorization (GRMF)  \cite{cai2011graph, rao2015collaborative, zhou2012kernelized} is: 
\begin{equation}\label{eq:grmf}
\min_{U,V} \sum_{(i,j) \in \Omega} \left(R_{i,j} - u_i^\top v_j \right)^2 + \frac{\lambda}{2} (\|U\|_F^2 + \|V\|_F^2) +\mu \trace(U^\top \text{Lap}(G) U)
\end{equation}
where $U \in \dR^{n \times r}, V \in \dR^{m \times r}$ are the embeddings associated with users and items respectively, $\trace$ is the trace operator, $\lambda, \mu$ are tuning coefficients, and $\text{Lap}(\cdot)$ is the Laplacian of $G$.  

The last term is called graph regularization, which
tries to enforce similar nodes (measured by edge weights in $G$) to have similar embeddings. 
One naive way \cite{cao2015grarep} to extend this to higher-order graph regularization is to replace the graph $G$ with $\sum_{i=1}^K w_i \cdot G^i$ and then use the graph Laplacian of $\sum_{i=1}^K w_i \cdot G^i$ to replace $G$ in \eqref{eq:grmf}. Computing $G^i$ for even small $i$ is computationally infeasible for most real-world applications, and we will soon lose the sparsity of the graph, leading to memory issues. 
Sampling or thresholding could mitigate the problem but suffers from performance degradation.

In contrast, our graph DNA $B$ from Algorithm~\ref{alg:graph-bloom} does not suffer from any of the issues. Theorem~\ref{thm:bf} implies that the space complexity of our method is only of order $O(n \log n)$ for a graph with $n$ nodes, instead of $O(n^2)$. The reduced number of non-zero elements using graph DNA leads to a significant speed-up in many cases.  


We can easily use graph DNA in GRMF as follows: we treat the $c$ bits as $c$ new pseudo-nodes and add them to the original graph $G$. We then have $n+c$ nodes in a modified graph $\dot{G}$: 

\begin{equation}\label{eq:G1}
\dot{G} =
\begin{bmatrix}
    G \in \dR^{n \times n}      & B\in \dR^{n \times c} \\
    B^\top\in \dR^{c \times n}       & \mathbf{0}\in \dR^{c \times c}
\end{bmatrix}.
\end{equation}

To account for the $c$ new nodes, we expand $U \in \dR^{n \times r}$ to $\dot{U} \in \dR^{(n+c) \times r}$ by appending parameters for the meta-nodes.
The objective function for GRMF with Graph DNA with be the same as \eqref{eq:grmf} except replacing $U$ and $G$ with $\dot{U}$ and $\dot{G}$.
At the prediction stage, we discard the meta-node embeddings. 


For implicit feedback data, when $R$ is a 0/1 matrix, weighted matrix factorization is a widely used algorithm~\cite{hu2008collaborative, hsieh2015pu}. The only difference is that the loss function in \eqref{eq:grmf} 
is replaced by $\sum_{(i, j): R_{ij}=1} (R_{ij}-u_i^T v_j)^2 + \sum_{(i, j): R_{ij}=0} \rho (R_{ij}-u_i^T v_j)^2$ where $\rho< 1$ is a hyper-parameter reflecting the confidence of zero entries. In this case, we can apply the Graph DNA encoding as before trivially. We also describe how to apply graph DNA towards Co-Factor \cite{singh2008relational, liang2016factorization} and Graph Convolutional Matrix Completion \cite{berg2017graph} in the Appendix.

\section{Experiments}
\label{sec:exp}
We show that our proposed Graph DNA encoding technique can improve the performance of 4 popular graph-based recommendation algorithms: graph-regularized matrix factorization, co-factorization, weighted matrix factorization, and GCN-based graph convolution matrix factorization.  
All experiments except GCN are conducted on a server with Intel Xeon E5-2699 v3 @ 2.30GHz CPU and 256$G$ RAM. The GCN experiments are conducted on Google Cloud with Nvidia V100 GPU.

\paragraph{Simulation Study}
\label{sec:exp-syn}

We first simulate a user/item rating dataset with user graph as side information, generate its graph DNA, and use it on a downstream task: matrix factorization. 

We randomly generate user and item embeddings from standard Gaussian distributions, and construct an Erd\H{o}s-R\'{e}nyi Random graphs of users. User embeddings are generated using Algorithm~\ref{alg:sim} in Appendix: at each propagation step, each user's embedding is updated by an average of its current embedding and its neighbors' embeddings.
Based on user and item embeddings after $T=3$ iterations of propagation,
we generate the underlying ratings for each user-item pairs according to the inner product of their embeddings, and then sample a small portion of the dense rating matrix as  training and test sets.




We implement our graph DNA encoding algorithm in python using a scalable python library \cite{almeida2007scalable}
to generate Bloom filter matrix $B$.  We adapt the GRMF C++ code to solve the objective function of GRMF\_DNA-K with our Bloom filter enhanced graph $\dot{G}$.
We compare the following variants:
\begin{enumerate}
    \item MF: classical matrix factorization only with $\ell_2$ regularization without graph information.
    \item GRMF\_$G^d$: GRMF with $\ell_2$ regularization and using $G$, $G^2$, \dots, $G^d$ \cite{cao2015grarep}.
    \item GRMF\_DNA-$d$: GRMF with $\ell_2$ but using our proposed graph DNA-$d$.
   
\end{enumerate}

We report the prediction performance with Root Mean Squared Error (RMSE) on test data.
All results are reported on the test set, with all relevant hyperparameters tuned on a held-out validation set. 
To accurately measure how large the relative gain is from using deeper information, we introduce a new metric called Relative Graph Gain (RGG) for using information $X$, which is defined as:
\begin{equation}\label{eq:rgg}
    \text{RGG}(X) = \left(\frac{\text{RMSE without Graph} - \text{RMSE with } X}{\text{RMSE without Graph} - \text{RMSE with } G} 
    - 1 \right) \times 100 \%, 
\end{equation} where RMSE is measured for the same method with different graph information. This metric would be 0 if only first order graph information is utilized and is only defined when the denominator is positive.

In Table~\ref{tab:mf_res}, we can easily see that using a deeper neighborhood helps the recommendation performances on this synthetic dataset. Graph DNA-3's gain is 166\%  larger than that of using first-order graph $G$. We can see an increase in performance gain for an increase in depth $d$ when $d \leq 3$. This is expected because we set $T = 3$ during our creation of this dataset.



\paragraph{Graph Regularized Matrix Factorization for Explicit Feedback}
\label{sec:exp-explicit}

Next, we show that graph DNA can improve the performance of GRMF for explicit feedback. We conduct experiments on two real datasets:  Douban  \cite{ma2011recommender} and Flixster  \cite{Zafarani+Liu:2009}.
Both datasets contain explicit feedback with ratings from 1 to 5. There are 129,490 users, 58,541 items in Douban. There are 147,612 users, 48,794 items in Flixster. Both datasets have a graph defined on the respective sets of users.  


We pre-processed Douban and Flixster following the same procedure in \cite{rao2015collaborative, wu2017large}. The experimental setups and comparisons are almost identical to the synthetic data experiment (see details in section~\ref{sec:exp-syn}). Due to the exponentially growing non-zero elements in the graph as we go deeper (see Table~\ref{tab:nnz}), we are unable to run full
GRMF\_$G^4$ and GRMF\_$G^5$ for these datasets. In fact,  GRMF\_$G^3$ itself is too slow so we thresholded $G^3$ by only considering entries whose values are equal to or larger than 4. For the Bloom filter, we set a false positive rate of 0.1 and use capacity of 500 for Bloom filters, resulting in $c = 4,796$. 


We can see from Table~\ref{tab:mf_res} that deeper graph information always helps. For Douban, graph DNA-3 is most effective, giving a relative graph gain of 82.79\% compared to only 2\% gain when using $G^2$ or $G^3$ naively. Interestingly for Flixster, using $G^2$ is better than using $G^3$. However, Graph DNA-3 and DNA-4 yield $10$x and $15$x performance improvements respectively, lending credence to the implicit regularization property of graph DNA. 
For a fixed size Bloom filter, the computational complexity of graph DNA scales linearly with depth $d$, as compared to exponentially for GRMF\_$G^d$.
We measure the speed in Table~\ref{tab:dnaspeed}. The memory cost is only a fraction of $n^2$ after hashing. Such low memory and computational complexity allow us to scale to larger $d$, compared to baseline methods.

\begin{table}
  \caption{Comparison of Graph Regularized Matrix Factorization Variants for Explicit Feedback on Synthetic, Douban and Flixster data. We use rank $r = 10$. RGG is the Relative Graph Gain in \eqref{eq:rgg}.} 
  \vskip -0.1in
  \label{tab:mf_res}
  \resizebox{\textwidth}{!}{
  \begin{tabular}{ccccccc}
    \toprule
    &\multicolumn{2}{c}{Synthetic} & \multicolumn{2}{c}{Douban} & \multicolumn{2}{c}{Flixster}\\
    \cmidrule(r){2-3} \cmidrule(r){4-5}  \cmidrule(r){6-7} 
    Dataset & RMSE ($ \times 10^{-1}$)  & \% RGG & RMSE ($ \times 10^{-1}$) & \% RGG & RMSE ($ \times 10^{-1}$)  & \% RGG \\
    \midrule
    
      MF        & 2.9971  & -  & 7.3107   & -  & 8.8111     & - \\
    GRMF\_$G$   & 2.7823 & 0  & 7.2398  & 0  & 8.8049   & 0\\
    GRMF\_$G^2$  & 2.6543  & 59.5903 & 7.2381  & 2.3977 & 8.7849  & 322.5806 \\
    GRMF\_$G^3$ & 2.5687  & 99.4413  & 7.2432  & -4.7954 & 8.7932  & 188.7097 \\
    GRMF\_$G^4$ & 2.5562 & 105.2607 & - & - & - & -\\
    GRMF\_$G^5$  & 2.4853  & 138.2682 & - & - & - & -\\
    GRMF\_$G^6$  & 2.4852  & 138.3147 & - & - & - & -\\
    GRMF\_DNA-1   & 2.4303  & 163.8734  & 7.2191  & 29.1960 & 8.8013   & 58.0645\\
    GRMF\_DNA-2  & 2.4510  & 154.2365 & 7.2359 & 5.5007 & 8.8007  & 67.7419 \\
    GRMF\_DNA-3  & \bfseries{2.4247}  &\bfseries{166.4804}  & \bfseries{7.1811}  & \bfseries{82.7927} & 8.7383  & 1074.1935\\
    GRMF\_DNA-4  & 2.4466   & 156.2849 & 7.1971 & 60.2257 & \bfseries{8.7122}  & \bfseries{1495.1613} \\
    \midrule
    Co-Factor\_$G$    & - & -   & 7.2743    & 0  & 8.7957    & 0   \\
    Co-Factor\_DNA-3  & - & - &\bfseries{7.2623}    & \bfseries{32.9670} & \bfseries{8.7354}   & \bfseries{391.5584} \\
    
    
  \bottomrule
\end{tabular}
}
\vskip -0.1in
\end{table}

\begin{table}
  \caption{Graph DNA (Algorithm~\ref{alg:graph-bloom}) Encoding Speed. We set number $c = 500$ and implement Graph DNA using single-core python. We can scale up linearly in terms of depth $d$ for a fixed $c$. } 
  \vskip -0.1in
  \label{tab:dnaspeed}
\resizebox{0.9\textwidth}{!}{
  \begin{tabular}{ccccccc}
    \toprule
    & \multicolumn{2}{c}{Graph Statistics} & \multicolumn{4}{c}{Graph DNA Encoding Time (secs)}\\
    \cmidrule(r){2-3} \cmidrule(r){4-7}
    Dataset & Number of Nodes & Graph Density & DNA-1 & DNA-2 & DNA-3 & DNA-4  \\
    \midrule
   
    Douban & 129,490 & 0.0102\% &132.2717 &   266.3740 & 403.9747 & 580.1547   \\
   Flixster & 147,612 & 0.0117\% & 157.3103  &  317.7706   & 482.0360   & 686.8048    \\
  \bottomrule
\end{tabular}
 }
\vskip -0.1in
\end{table}

\begin{table}
  \caption{Comparison of GRWMF Variants for Implicit Feedback on Douban and Flixster datasets. P stands for precision and N stands for NDCG. We use rank $r = 10$ and all results are in $\%$.}
  \vskip -0.1in
  \label{tab:implicit_res}
   \resizebox{0.8\textwidth}{!}{
  \begin{tabular}{llccccccc}
    \toprule
    Dataset & Methods & MAP & HLU & P@$1$ & P@$5$ & N@$1$ & N@$5$  \\
    \midrule
    \multirow{2}{*}{Douban}
    &GRWMF\_$G$               & 8.340 & 13.033 & 14.944 & 10.371 & 14.944 & 12.564    \\
    &GRWMF\_DNA-3 & \bfseries{8.400} & \bfseries{13.110} & \bfseries{14.991} & \bfseries{10.397} & \bfseries{14.991} & \bfseries{12.619}     \\
     \midrule
    \multirow{2}{*}{Flixster}
    &GRWMF\_$G$              & 10.889 & 14.909 & 12.303 & 7.9927 & 12.303 & 12.734   \\
     &GRWMF\_DNA-3 & \bfseries{11.612} & \bfseries{15.687} & \bfseries{12.644} & \bfseries{8.1583} & \bfseries{12.644} & \bfseries{13.399}    \\
  \bottomrule
\end{tabular}
 }
\end{table}




\begin{table}
  \caption{Comparison of GCN Methods for Explicit Feedback on Douban, Flixster and Yahoo Music datasets (3000 by 3000 as in \cite{berg2017graph, monti2017geometric}). All the methods except GC-MC utilize side graph information. } 
  \vskip -0.1in
  \label{tab:gcn}
  \resizebox{\textwidth}{!}{
  \begin{tabular}{llcccc}
    \toprule
    Dataset & Methods & Test RMSE ($ \times 10^{-1}$) & Time/epoch (secs)  & \% RGG & Speedup  \\
    \midrule
    \multirow{2}{*}{Douban}
     & SRGCNN (reported by \cite{berg2017graph}) & 8.0100  & - & - & - \\
     &GC-MC                  &  7.3109 $\pm$ 0.0150   & \bfseries{0.0410} & -  &9.72x \\
    &GC-MC\_$G$                    & 7.3698 $\pm$  0.0737  & 0.3985 & N/A  &  1.00x \\
    &GC-MC\_$G^2$                 & 7.3123  $\pm$  0.0139 & 0.4221 & N/A  &  0.94x   \\
    &GC-MC\_DNA-2  & \bfseries{7.3117} $\pm$ \bfseries{0.0129}     & 0.1709 & N/A  & 2.33x \\
    \midrule
    \multirow{2}{*}{Flixster}
    & SRGCNN (reported by \cite{berg2017graph}) & 9.2600  & - & - & - \\
     &GC-MC                   & 9.2614  $\pm$ 0.0578  & \bfseries{0.0232}   & -   & \bfseries{13.65x}   \\
     &GC-MC\_$G$                    &  9.2374 $\pm$  0.1045 & 0.3166   & 0   & 1.00x   \\
      &GC-MC\_$G^2$                    & \bfseries{8.9344} $\pm$  \bfseries{0.0333} &  0.3291  &   \bfseries{1262.4999} & 0.96x  \\
    &GC-MC\_DNA-2  & 8.9536  $\pm$ 0.0770   & 0.0524 & 1182.4999 & 6.04x \\
    \midrule
    \multirow{2}{*}{Yahoo Music}
     & SRGCNN (reported by \cite{berg2017graph}) & 22.4000 & - & - & - \\
     
      &GC-MC  &  22.6697 $\pm$  0.3530 &  \bfseries{0.0684}  & -   & \bfseries{1.75x}  \\
     &GC-MC\_$G$  &  21.3672 $\pm$  0.4190 & 0.1198   & 0   & 1.00x\\
    &GC-MC\_$G^2$    & 20.2189 $\pm$  0.8664 & 0.1177 & 88.1612 & 1.02x     \\
    &GC-MC\_DNA-2  & \bfseries{19.3879} $\pm$ \bfseries{0.2874}     &  0.0896 & \bfseries{151.9616} & 1.34x \\
  \bottomrule
\end{tabular}
}
\vskip -0.1in
\end{table}

\paragraph{Co-Factorization with Graph for Explicit Feedback}\label{sec:co-factor} 
We show our graph DNA can improve Co-Factor \cite{singh2008relational, liang2016factorization} as well. 
The results are in Table~\ref{tab:mf_res}. We find that applying DNA-3 to the Co-Factor method improves performance on both the datasets, more so for Flixster. 
This is consistent with our observations for GRMF in Table~\ref{tab:mf_res}: deep graph information is more helpful for Flixster than Douban. Applying Graph DNA to Co-Factor is detailed in the Appendix.

\paragraph{Graph Regularized Weighted Matrix Factorization for Implicit feedback}
\label{sec:exp-implicit}
We follow the same procedure as in \cite{wu2018sql} to set ratings of 4 and above to 1, and the rest to 0. 
We compare the baseline graph based weighted matrix factorization \cite{hu2008collaborative, hsieh2015pu} with our proposed weighted matrix factorization with DNA-3. We do not compare with Bayesian personalized ranking \cite{rendle2009bpr} and the recently proposed SQL-rank \cite{wu2018sql} 
as they cannot easily utilize graph information.



The results are summarized in Table~\ref{tab:implicit_res} with experimental details in the Appendix. Again, using DNA-3 achieves better prediction results over the baseline in terms of every single metric on both Douban and Flixster datasets.

\paragraph{Graph Convolutional Matrix Factorization}
Graph Convolutional Matrix Completion (GC-MC) is a graph convolutional network (GCN) based geometric matrix completion method \cite{berg2017graph}. In \cite{berg2017graph}, the side graphs over users and items are represented as the adjacency matrices and these one-hot encodings are treated as features for nodes in the graph. Convolutions of these features are performed on the bipartite rating graph. We find in our experiments that using these one-hot encodings of the graph as feature is an inferior choice both in terms of performance and speed. To capture higher order side graph information, it is better to use $G + \alpha G^2$ for some constant $\alpha$. 
Again, we can use graph DNA instead to efficiently encode and store the higher order information before feeding it into GC-MC. The exact means to use Graph DNA is detailed in the Appendix.

\begin{wrapfigure}{r}{0.5\textwidth}
\vspace{-2em}
  \begin{center}
    \includegraphics[width=\textwidth]{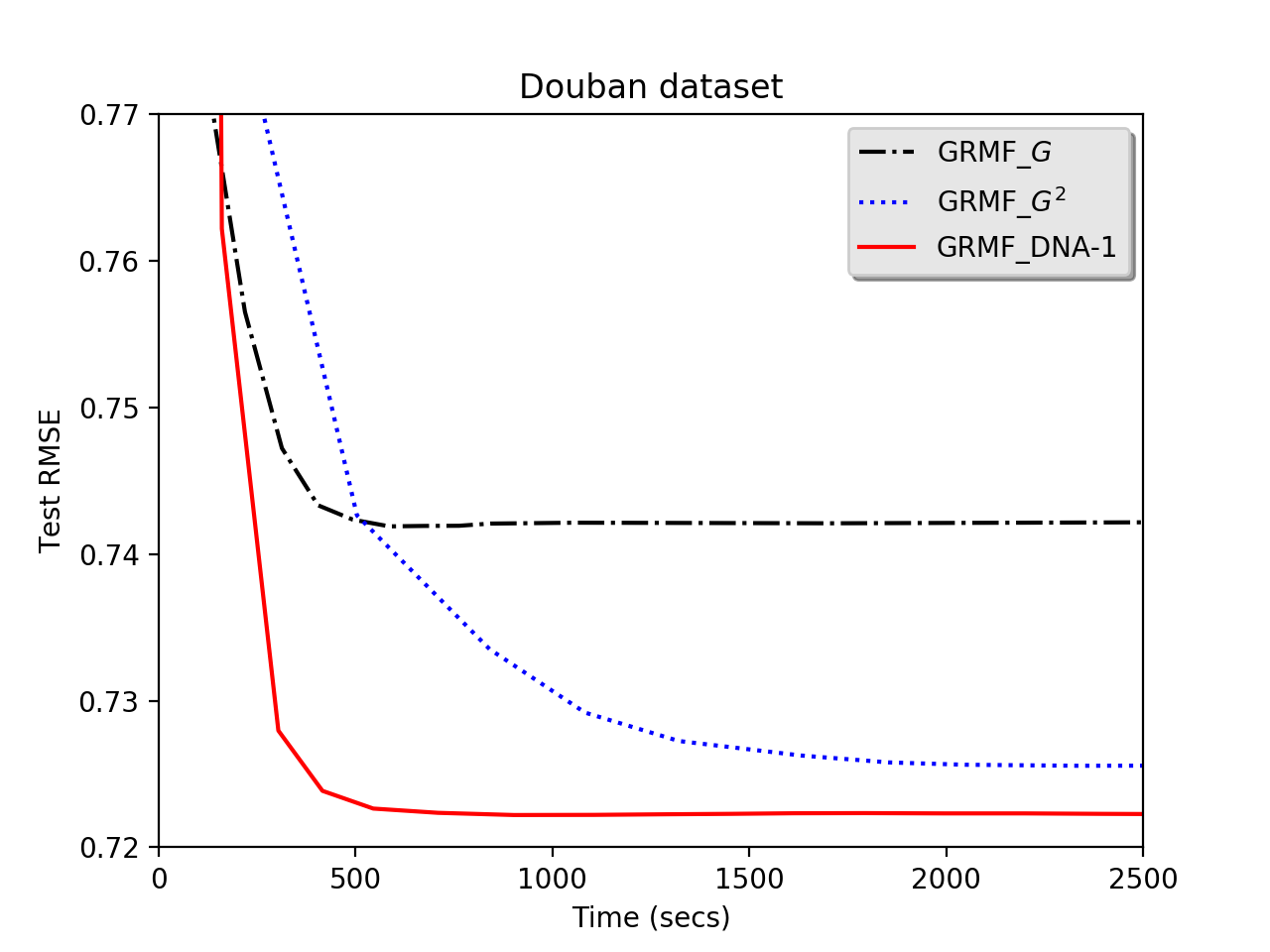}
  \end{center}
  \vspace{-5pt}
  \caption{Compare Training Speed of GRMF, with and without Graph DNA.}
  \vspace{-2em}
  \label{fig:speed}
\end{wrapfigure}
We use the same split of three real-world datasets and follow the exact procedures as in \cite{berg2017graph, monti2017geometric}.
We tuned hyperparameters using a validation dataset and obtain the best test results found within 200 epochs using optimal parameters. We repeated the experiments 6 times and report the mean and standard deviation of test RMSE. After some tuning, we use the capacity of 10 Bloom filters for Douban and 60 for Flixster, as the latter has a much denser second-order graph. With a false positive rate of 0.1, this implies that we use 96-bits Bloom filters for Douban and 960 bits for Flixster. So the feature dimension is reduced from 3000 to 96 and 960 when using our graph DNA-2, which leads to a significant speed-up. The original GC-MC method did not scale up well beyond 3000 by 3000 rating matrices with the user and the item side graphs as it requires using normalized adjacency matrix as user/item features. PinSage \cite{ying2018graph}, while scalable, does not utilize the user/item side graphs. Furthermore, it is not feasible to have $O(n)$ dimensional features for the nodes, where $n$ is the number of nodes in side graphs. By contrast, our method only requires $O(\log(n))$ dimensional features. We can see from Table~\ref{tab:gcn} that we outperform both GCN-based methods \cite{berg2017graph} and \cite{monti2017geometric} in terms of speed and performance by a large margin. 
\vspace{-0.5em}
\paragraph{Speed Comparisons}\label{sec:speed}
Finally, we compare the speed-ups obtained by graph DNA-$d$ with GRMF $G^d$. Since both algorithms scale with the number of edges in the constructed graph, we see that the Bloom filter based method scales substantially better compared to computing and using $G^2$ in Figure~\ref{fig:speed}.

\vspace{-1em}
 \section{Conclusion}
 \label{sec:conclusion}
\vspace{-1em}
In this paper, we proposed Graph DNA, a deep neighborhood aware encoding scheme for collaborative filtering with graph information. We make use of Bloom filters to incorporate higher order graph information, without the need to explicitly minimize a loss function. The resulting encoding is extremely space and computationally efficient, and lends itself well to multiple algorithms that make use of graph information, including Graph Convolutional Networks. Experiments show that Graph DNA encoding outperforms several baseline methods on multiple datasets in both speed and performance.

%


\medskip
\bibliography{neurips_2019}

\begin{thebibliography}{45}
\providecommand{\natexlab}[1]{#1}
\providecommand{\url}[1]{\texttt{#1}}
\expandafter\ifx\csname urlstyle\endcsname\relax
  \providecommand{\doi}[1]{doi: #1}\else
  \providecommand{\doi}{doi: \begingroup \urlstyle{rm}\Url}\fi

\bibitem[Abbassi and Mirrokni(2007)]{abbassi2007recommender}
Zeinab Abbassi and Vahab~S Mirrokni.
\newblock A recommender system based on local random walks and spectral
  methods.
\newblock In \emph{Proceedings of the 9th WebKDD and 1st SNA-KDD 2007 workshop
  on Web mining and social network analysis}, pages 102--108. ACM, 2007.

\bibitem[Almeida et~al.(2007)Almeida, Baquero, Pregui{\c{c}}a, and
  Hutchison]{almeida2007scalable}
Paulo~S{\'e}rgio Almeida, Carlos Baquero, Nuno Pregui{\c{c}}a, and David
  Hutchison.
\newblock Scalable bloom filters.
\newblock \emph{Information Processing Letters}, 101\penalty0 (6):\penalty0
  255--261, 2007.

\bibitem[Berg et~al.(2017)Berg, Kipf, and Welling]{berg2017graph}
Rianne van~den Berg, Thomas~N Kipf, and Max Welling.
\newblock Graph convolutional matrix completion.
\newblock \emph{arXiv preprint arXiv:1706.02263}, 2017.

\bibitem[Bloom(1970)]{bloom1970space}
Burton~H Bloom.
\newblock Space/time trade-offs in hash coding with allowable errors.
\newblock \emph{Communications of the ACM}, 13\penalty0 (7):\penalty0 422--426,
  1970.

\bibitem[Borthakur et~al.(2011)Borthakur, Gray, Sarma, Muthukkaruppan,
  Spiegelberg, Kuang, Ranganathan, Molkov, Menon, Rash,
  et~al.]{borthakur2011apache}
Dhruba Borthakur, Jonathan Gray, Joydeep~Sen Sarma, Kannan Muthukkaruppan,
  Nicolas Spiegelberg, Hairong Kuang, Karthik Ranganathan, Dmytro Molkov,
  Aravind Menon, Samuel Rash, et~al.
\newblock Apache hadoop goes realtime at facebook.
\newblock In \emph{Proceedings of the 2011 ACM SIGMOD International Conference
  on Management of data}, pages 1071--1080. ACM, 2011.

\bibitem[Breese et~al.(1998)Breese, Heckerman, and Kadie]{breese1998empirical}
John~S Breese, David Heckerman, and Carl Kadie.
\newblock Empirical analysis of predictive algorithms for collaborative
  filtering.
\newblock In \emph{Proceedings of the Fourteenth conference on Uncertainty in
  artificial intelligence}, pages 43--52. Morgan Kaufmann Publishers Inc.,
  1998.

\bibitem[Broder and Mitzenmacher(2004)]{broder2004network}
Andrei Broder and Michael Mitzenmacher.
\newblock Network applications of bloom filters: A survey.
\newblock \emph{Internet mathematics}, 1\penalty0 (4):\penalty0 485--509, 2004.

\bibitem[Cai et~al.(2011)Cai, He, Han, and Huang]{cai2011graph}
Deng Cai, Xiaofei He, Jiawei Han, and Thomas~S Huang.
\newblock Graph regularized nonnegative matrix factorization for data
  representation.
\newblock \emph{IEEE Transactions on Pattern Analysis and Machine
  Intelligence}, 33\penalty0 (8):\penalty0 1548--1560, 2011.

\bibitem[Cao et~al.(2015)Cao, Lu, and Xu]{cao2015grarep}
Shaosheng Cao, Wei Lu, and Qiongkai Xu.
\newblock Grarep: Learning graph representations with global structural
  information.
\newblock In \emph{Proceedings of the 24th ACM international on conference on
  information and knowledge management}, pages 891--900. ACM, 2015.

\bibitem[Chang et~al.(2008)Chang, Dean, Ghemawat, Hsieh, Wallach, Burrows,
  Chandra, Fikes, and Gruber]{chang2008bigtable}
Fay Chang, Jeffrey Dean, Sanjay Ghemawat, Wilson~C Hsieh, Deborah~A Wallach,
  Mike Burrows, Tushar Chandra, Andrew Fikes, and Robert~E Gruber.
\newblock Bigtable: A distributed storage system for structured data.
\newblock \emph{ACM Transactions on Computer Systems (TOCS)}, 26\penalty0
  (2):\penalty0 4, 2008.

\bibitem[Chen et~al.(2018)Chen, Zhu, and Song]{chen2018stochastic}
Jianfei Chen, Jun Zhu, and Le~Song.
\newblock Stochastic training of graph convolutional networks with variance
  reduction.
\newblock In \emph{International Conference on Machine Learning}, pages
  941--949, 2018.

\bibitem[Cisse et~al.(2013)Cisse, Usunier, Artieres, and
  Gallinari]{cisse2013robust}
Moustapha~M Cisse, Nicolas Usunier, Thierry Artieres, and Patrick Gallinari.
\newblock Robust bloom filters for large multilabel classification tasks.
\newblock In \emph{Advances in Neural Information Processing Systems}, pages
  1851--1859, 2013.

\bibitem[Courbariaux et~al.(2015)Courbariaux, Bengio, and
  David]{courbariaux2015binaryconnect}
Matthieu Courbariaux, Yoshua Bengio, and Jean-Pierre David.
\newblock Binaryconnect: Training deep neural networks with binary weights
  during propagations.
\newblock In \emph{Advances in neural information processing systems}, pages
  3123--3131, 2015.

\bibitem[Defferrard et~al.(2016)Defferrard, Bresson, and
  Vandergheynst]{defferrard2016convolutional}
Micha{\"e}l Defferrard, Xavier Bresson, and Pierre Vandergheynst.
\newblock Convolutional neural networks on graphs with fast localized spectral
  filtering.
\newblock In \emph{Advances in Neural Information Processing Systems}, pages
  3844--3852, 2016.

\bibitem[Dubhashi and Ranjan(1998)]{dubhashi1998balls}
Devdatt Dubhashi and Desh Ranjan.
\newblock Balls and bins: A study in negative dependence.
\newblock \emph{Random Structures \& Algorithms}, 13\penalty0 (2):\penalty0
  99--124, 1998.

\bibitem[Gori et~al.(2007)Gori, Pucci, Roma, and Siena]{gori2007itemrank}
Marco Gori, Augusto Pucci, V~Roma, and I~Siena.
\newblock Itemrank: A random-walk based scoring algorithm for recommender
  engines.
\newblock In \emph{IJCAI}, volume~7, pages 2766--2771, 2007.

\bibitem[Hamilton et~al.(2017{\natexlab{a}})Hamilton, Ying, and
  Leskovec]{hamilton2017inductive}
Will Hamilton, Zhitao Ying, and Jure Leskovec.
\newblock Inductive representation learning on large graphs.
\newblock In \emph{Advances in Neural Information Processing Systems}, pages
  1024--1034, 2017{\natexlab{a}}.

\bibitem[Hamilton et~al.(2017{\natexlab{b}})Hamilton, Ying, and
  Leskovec]{hamilton2017representation}
William~L Hamilton, Rex Ying, and Jure Leskovec.
\newblock Representation learning on graphs: Methods and applications.
\newblock \emph{arXiv preprint arXiv:1709.05584}, 2017{\natexlab{b}}.

\bibitem[Han et~al.(2015)Han, Mao, and Dally]{han2015deep}
Song Han, Huizi Mao, and William~J Dally.
\newblock Deep compression: Compressing deep neural networks with pruning,
  trained quantization and huffman coding.
\newblock \emph{arXiv preprint arXiv:1510.00149}, 2015.

\bibitem[Hsieh et~al.(2015)Hsieh, Natarajan, and Dhillon]{hsieh2015pu}
Cho-Jui Hsieh, Nagarajan Natarajan, and Inderjit Dhillon.
\newblock Pu learning for matrix completion.
\newblock In \emph{International Conference on Machine Learning}, pages
  2445--2453, 2015.

\bibitem[Hu et~al.(2008)Hu, Koren, and Volinsky]{hu2008collaborative}
Yifan Hu, Yehuda Koren, and Chris Volinsky.
\newblock Collaborative filtering for implicit feedback datasets.
\newblock In \emph{Data Mining, 2008. ICDM'08. Eighth IEEE International
  Conference on}, pages 263--272. Ieee, 2008.

\bibitem[Jamali and Ester(2009)]{jamali2009trustwalker}
Mohsen Jamali and Martin Ester.
\newblock Trustwalker: a random walk model for combining trust-based and
  item-based recommendation.
\newblock In \emph{Proceedings of the 15th ACM SIGKDD international conference
  on Knowledge discovery and data mining}, pages 397--406. ACM, 2009.

\bibitem[Joag-Dev et~al.(1983)Joag-Dev, Proschan, et~al.]{joag1983negative}
Kumar Joag-Dev, Frank Proschan, et~al.
\newblock Negative association of random variables with applications.
\newblock \emph{The Annals of Statistics}, 11\penalty0 (1):\penalty0 286--295,
  1983.

\bibitem[Kipf and Welling(2016)]{kipf2016semi}
Thomas~N Kipf and Max Welling.
\newblock Semi-supervised classification with graph convolutional networks.
\newblock \emph{arXiv preprint arXiv:1609.02907}, 2016.

\bibitem[Koren et~al.(2009)Koren, Bell, and Volinsky]{koren2009matrix}
Yehuda Koren, Robert Bell, and Chris Volinsky.
\newblock Matrix factorization techniques for recommender systems.
\newblock \emph{Computer}, \penalty0 (8):\penalty0 30--37, 2009.

\bibitem[Liang et~al.(2016)Liang, Altosaar, Charlin, and
  Blei]{liang2016factorization}
Dawen Liang, Jaan Altosaar, Laurent Charlin, and David~M Blei.
\newblock Factorization meets the item embedding: Regularizing matrix
  factorization with item co-occurrence.
\newblock In \emph{Proceedings of the 10th ACM conference on recommender
  systems}, pages 59--66. ACM, 2016.

\bibitem[Ma et~al.(2011)Ma, Zhou, Liu, Lyu, and King]{ma2011recommender}
Hao Ma, Dengyong Zhou, Chao Liu, Michael~R Lyu, and Irwin King.
\newblock Recommender systems with social regularization.
\newblock In \emph{Proceedings of the fourth ACM international conference on
  Web search and data mining}, pages 287--296. ACM, 2011.

\bibitem[Monti et~al.(2017)Monti, Bronstein, and Bresson]{monti2017geometric}
Federico Monti, Michael Bronstein, and Xavier Bresson.
\newblock Geometric matrix completion with recurrent multi-graph neural
  networks.
\newblock In \emph{Advances in Neural Information Processing Systems}, pages
  3697--3707, 2017.

\bibitem[Page et~al.(1999)Page, Brin, Motwani, and Winograd]{page1999pagerank}
Lawrence Page, Sergey Brin, Rajeev Motwani, and Terry Winograd.
\newblock The pagerank citation ranking: Bringing order to the web.
\newblock Technical report, Stanford InfoLab, 1999.

\bibitem[Pozo et~al.(2016)Pozo, Chiky, Meziane, and M{\'e}tais]{pozo2016item}
Manuel Pozo, Raja Chiky, Farid Meziane, and Elisabeth M{\'e}tais.
\newblock An item/user representation for recommender systems based on bloom
  filters.
\newblock In \emph{2016 IEEE Tenth International Conference on Research
  Challenges in Information Science (RCIS)}, pages 1--12. IEEE, 2016.

\bibitem[Rao et~al.(2015)Rao, Yu, Ravikumar, and Dhillon]{rao2015collaborative}
Nikhil Rao, Hsiang-Fu Yu, Pradeep~K Ravikumar, and Inderjit~S Dhillon.
\newblock Collaborative filtering with graph information: Consistency and
  scalable methods.
\newblock In \emph{Advances in neural information processing systems}, pages
  2107--2115, 2015.

\bibitem[Rendle et~al.(2009)Rendle, Freudenthaler, Gantner, and
  Schmidt-Thieme]{rendle2009bpr}
Steffen Rendle, Christoph Freudenthaler, Zeno Gantner, and Lars Schmidt-Thieme.
\newblock Bpr: Bayesian personalized ranking from implicit feedback.
\newblock In \emph{Proceedings of the twenty-fifth conference on uncertainty in
  artificial intelligence}, pages 452--461. AUAI Press, 2009.

\bibitem[Serr{\`a} and Karatzoglou(2017)]{serra2017getting}
Joan Serr{\`a} and Alexandros Karatzoglou.
\newblock Getting deep recommenders fit: Bloom embeddings for sparse binary
  input/output networks.
\newblock In \emph{Proceedings of the Eleventh ACM Conference on Recommender
  Systems}, pages 279--287. ACM, 2017.

\bibitem[Shah et~al.(2009)]{shah2009gossip}
Devavrat Shah et~al.
\newblock Gossip algorithms.
\newblock \emph{Foundations and Trends{\textregistered} in Networking},
  3\penalty0 (1):\penalty0 1--125, 2009.

\bibitem[Shani et~al.(2008)Shani, Chickering, and Meek]{shani2008mining}
Guy Shani, Max Chickering, and Christopher Meek.
\newblock Mining recommendations from the web.
\newblock In \emph{Proceedings of the 2008 ACM conference on Recommender
  systems}, pages 35--42. ACM, 2008.

\bibitem[Shi et~al.(2009)Shi, Petterson, Dror, Langford, Smola, and
  Vishwanathan]{shi2009hash}
Qinfeng Shi, James Petterson, Gideon Dror, John Langford, Alex Smola, and SVN
  Vishwanathan.
\newblock Hash kernels for structured data.
\newblock \emph{Journal of Machine Learning Research}, 10\penalty0
  (Nov):\penalty0 2615--2637, 2009.

\bibitem[Shinde and Savant(2016)]{shinde2016user}
Anita Shinde and Ila Savant.
\newblock User based collaborative filtering using bloom filter with mapreduce.
\newblock In \emph{Proceedings of International Conference on ICT for
  Sustainable Development}, pages 115--123. Springer, 2016.

\bibitem[Singh and Gordon(2008)]{singh2008relational}
Ajit~P Singh and Geoffrey~J Gordon.
\newblock Relational learning via collective matrix factorization.
\newblock In \emph{Proceedings of the 14th ACM SIGKDD international conference
  on Knowledge discovery and data mining}, pages 650--658. ACM, 2008.

\bibitem[Wu et~al.(2017)Wu, Hsieh, and Sharpnack]{wu2017large}
Liwei Wu, Cho-Jui Hsieh, and James Sharpnack.
\newblock Large-scale collaborative ranking in near-linear time.
\newblock In \emph{Proceedings of the 23rd ACM SIGKDD International Conference
  on Knowledge Discovery and Data Mining}, pages 515--524. ACM, 2017.

\bibitem[Wu et~al.(2018)Wu, Hsieh, and Sharpnack]{wu2018sql}
Liwei Wu, Cho-Jui Hsieh, and James Sharpnack.
\newblock Sql-rank: A listwise approach to collaborative ranking.
\newblock In \emph{Proceedings of Machine Learning Research (35th International
  Conference on Machine Learning)}, volume~80, 2018.

\bibitem[Xie et~al.(2015)Xie, Bindel, Demers, and Gehrke]{xie2015edge}
Wenlei Xie, David Bindel, Alan Demers, and Johannes Gehrke.
\newblock Edge-weighted personalized pagerank: breaking a decade-old
  performance barrier.
\newblock In \emph{Proceedings of the 21th ACM SIGKDD International Conference
  on Knowledge Discovery and Data Mining}, pages 1325--1334. ACM, 2015.

\bibitem[Ying et~al.(2018)Ying, He, Chen, Eksombatchai, Hamilton, and
  Leskovec]{ying2018graph}
Rex Ying, Ruining He, Kaifeng Chen, Pong Eksombatchai, William~L Hamilton, and
  Jure Leskovec.
\newblock Graph convolutional neural networks for web-scale recommender
  systems.
\newblock In \emph{Proceedings of the 24th ACM SIGKDD International Conference
  on Knowledge Discovery \& Data Mining}, pages 974--983. ACM, 2018.

\bibitem[Yu et~al.(2017)Yu, Huang, Dhillon, and Lin]{yu2017unified}
Hsiang-Fu Yu, Hsin-Yuan Huang, Inderjit~S Dhillon, and Chih-Jen Lin.
\newblock A unified algorithm for one-class structured matrix factorization
  with side information.
\newblock In \emph{AAAI}, pages 2845--2851, 2017.

\bibitem[Zafarani and Liu(2009)]{Zafarani+Liu:2009}
R.~Zafarani and H.~Liu.
\newblock Social computing data repository at {ASU}, 2009.
\newblock URL \url{http://socialcomputing.asu.edu}.

\bibitem[Zhou et~al.(2012)Zhou, Shan, Banerjee, and Sapiro]{zhou2012kernelized}
Tinghui Zhou, Hanhuai Shan, Arindam Banerjee, and Guillermo Sapiro.
\newblock Kernelized probabilistic matrix factorization: Exploiting graphs and
  side information.
\newblock In \emph{Proceedings of the 2012 SIAM international Conference on
  Data mining}, pages 403--414. SIAM, 2012.

\end{thebibliography}
\small

\newpage
\section{Appendix}

\subsection{Theory for Bloom Filters}
\label{sec:theory}

\begin{theorem}
\label{thm:main}
Let $B_x, B_y$ be the Bloom filter bitarrays for $N(x), N(y)$ with independent hash functions for all elements of $N(x) \cup N(y)$ and $|N(x) \triangle N(y)|$ be their symmetric difference.
Let $Q$ be the number of common 1-bits in $B_x, B_y$, then we have that,
\begin{align*}
\mathbb P \left\{ Q \ge (1 + \delta) \mathbb E Q \right\} \le \left( \frac{e^\delta}{(1 + \delta)^{(1 + \delta)}}  \right)^{\mathbb E Q},\\
\mathbb P \left\{ Q \le (1 - \delta) \mathbb E Q \right\} \le \left( \frac{e^{-\delta}}{(1 - \delta)^{(1 - \delta)}}  \right)^{\mathbb E Q},
\end{align*}
and $\Gamma_0 \le \mathbb E Q \le \Gamma_1$ where
\begin{align*}
&\Gamma_0 = c \left( 1 - \exp\left\{ - k \frac{|N(x) \cap N(y)|}{c} \right\} \right), \\
&\Gamma_1 = c \left( 1 - \exp\left\{ - k^2 \frac{|N(x) \triangle N(y)|^2}{4c^2} - \frac{k |N(x) \cap N(y)|}{c-1} \right\}\right).
\end{align*}
\end{theorem}

We prove Theorem \ref{thm:main} in subsection \ref{sec:main_proof}.  For now, we prove Theorem \ref{thm:bf} which is in fact a corollary of this main result.

\begin{proof}[Proof of Theorem \ref{thm:bf}]
We can see that there exist $C_0,C_1$ such that for any $\delta \ge 0$,
\[
\left( \frac{e^\delta}{(1 + \delta)^{(1 + \delta)}}  \right)^{\mathbb E Q} \le C_1 e^{- C_0 \delta \mathbb E Q}.
\]
Then we have that with probability $1 - \gamma$,
\[
Q \le \left( 1 + \frac{1}{C_0} \log \frac{C_1}{\gamma} \right) \mathbb E Q \le \left( 1 + \frac{1}{C_0} \log \frac{C_1}{\gamma} \right) \Gamma_1.
\]
Note that because $1 - e^{-x} \le x$,
\[
\Gamma_1 \le k^2 \frac{|N(x) \triangle N(y)|^2}{4c} + \frac{c k |N(x) \cap N(y)|}{c-1}.
\]
Moreover, for $\delta \in (0,1)$,
\[
\left( \frac{e^{-\delta}}{(1 - \delta)^{(1 - \delta)}}  \right)^{\mathbb E Q} \le e^{- \frac{\delta^2}{3} \mathbb E Q} \le e^{- \frac{\delta^2}{3} \Gamma_0}.
\]
Hence, 
\[
\mathbb P \left\{ Q \le (1 - \delta) \Gamma_0 \right\} \le  e^{- \frac{\delta^2}{3} \Gamma_0}.
\]
Suppose that for some $\alpha \in (0,1)$, $\alpha c > k |N(x) \cap N(y)|$ then we have that 
\[
\Gamma_0 \ge \frac{(1 - e^{-\alpha})}{\alpha} k |N(x) \cap N(y)|.
\]
The function $(1 - e^{-\alpha}) / \alpha$ is decreasing and the limit as $\alpha \rightarrow 0$ is $1$.
Thus, for any $\delta \in (0,1)$, there exists an $\alpha \ge 0$ such that if $\alpha c > k |N(x) \cap N(y)|$ then $\Gamma_0 \ge (1 - \delta) k |N(x) \cap N(y)|$.
If this is the case then
\[
\mathbb P \left\{ Q \le (1 - \delta) k |N(x) \cap N(y)| \right\} \le  e^{- \frac 13 (1 - \delta) \delta^2 k |N(x) \cap N(y)|}.
\]
\end{proof}

\subsubsection{Negative Associativity of Bloom Filters}

First, let us go over the definition of negative associativity.
Random variables, $\{q_i\}_{i=1}^c$, are negatively associative (NA), if for any functions $f,g$, both monotonically increasing or decreasing, and disjoint sets $I,J \subset \{1,\ldots,c\}$,
\[
\mathbb E [f(q_I) g(q_J)] \le \mathbb E[f(q_I)] \cdot \mathbb E[g(q_J)],
\]
where $q_I,q_J$ are the variables restricted to these sets.

\begin{lemma}
\label{lem:NA}
\noindent {\bf (1)} Let $\{ q_{0,i} \}_{i = 1}^c, \{ q_{1,i} \}_{i = 1}^c \subset \{0,1\}^c$ be two independent random bitarrays that are both NA.  Then $q_0 | q_1$, the elementwise `or' operation, and $q_0 \& q_1$, the elementwise `and' operation, are both NA.  So NA is closed under elementwise `or' and `and' operations.

\noindent {\bf (2)} Let $q_i$ be the $i$th bit in any Bloom filter of the set $\mathcal N$ with independent hash functions, then the random bits, $\{q_i\}_{i=1}^c$, are NA.
\end{lemma}

\begin{proof} {\bf (1)} 
We show that NA is closed under both elementwise operations.
First, note that the concatenation $\{q_{0,1}, \ldots, q_{0,c}, q_{1,1},\ldots, q_{1,c}$ is NA, by closure of NA under independent union (Property P7 in \cite{joag1983negative}).
Then on the disjoint sets, $\{ q_{0,i}, q_{1,i} \}_{i=1}^c$, apply the bit operation to produce the resulting array.
Operation `or' is monotonically increasing because, $q_{0,i} | q_{1,i} = 1\{ q_{0,i} + q_{1,i} > 0 \}$, `and' is as well because $q_{0,i} \& q_{1,i} = 1\{ q_{0,i} + q_{1,i} > 1 \}$.
Finally we conclude by closure of NA under monotonic increasing functions on disjoint sets (Property P6 in \cite{joag1983negative}).

{\bf (2)}
Consider hash function $j \in \{1,\ldots,k\}$ for node $v \in \mathcal N$.
Let $B^j_v$ be the $c$-bit Bloom filter bitarray for this vertex and hash function only, then $B^j_v$ has only a single bit that is $1$ and the rest are $0$.
By the 0-1 property for binary bits, we know that $B^j_v$ has NA entries (Lemma 8 in \cite{dubhashi1998balls}), since $\sum_i B^j_{v,i} = 1$.
Then the Bloom filter, $B$, of $\mathcal N$ is $B = |_{j = 1}^k |_{v \in N(x)} B^j_v$---the `or' operation applied to all hashes and vertices, and we conclude by property (1).
\end{proof}

\subsubsection{Proof of Theorem \ref{thm:main}}
\label{sec:main_proof}

Consider the partition of $N(x) \cup N(y)$ into $A_1 = N(x) \backslash N(y)$, $A_2 = N(y) \backslash N(x)$, $A_3 = N(x) \cap N(y)$.
Let $B_x, B_y$ be the Bloom filter bitarrays for $N(x), N(y)$ and let $B_1,B_2,B_3$ be those for $A_1, A_2,A_3$ respectively.

Notice that $B_x \& B_y = B_3 | (B_1 \& B_2)$, where the bit operations are elementwise.  If all hash functions are independent, then $B_1, B_2, B_3$ are independent.
Notice that for a given node and hash function the bit selected is random, but unique, which means that the elements of the bitarrays are not necessarily independent for any Bloom filter.
However, the bitarray $B_3 | (B_1 \& B_2)$ is negatively associative by Lemma \ref{lem:NA}.
Let $q_i = (B_{1,i} \& B_{2,i}) | B_{3,i}$, then we have that,
\[
\mathbb E q_i = 1 - (1 - \mathbb E [B_{1,i}] \cdot\mathbb E [B_{2,i}]) (1 - \mathbb E [B_{3,i}]).
\]
The probability that bit $i$ in one of the bitarrays is $0$ is 
\[
1 - \mathbb E [B_{j,i}] = \left( 1 - \frac 1c \right)^{k |A_j|}, \quad j=1,2,3.
\]
This can give us an expression in terms of $c,k,|A_1|,|A_2|,|A_3|$ for the expectation of $Q = \sum_{i=1}^c q_i$.
We have that by Hoeffding's inequality for negatively associative random variables \cite{dubhashi1998balls}, 
\[
\mathbb P \left\{ Q \ge (1 + \delta) \mathbb E Q \right\} \le \left( \frac{e^\delta}{(1 + \delta)^{(1 + \delta)}}  \right)^{\mathbb E Q},
\]
\[
\mathbb P \left\{ Q \le (1 - \delta) \mathbb E Q \right\} \le \left( \frac{e^{-\delta}}{(1 - \delta)^{(1 - \delta)}}  \right)^{\mathbb E Q}.
\]

It remains to provide intelligible bounds on $\mathbb E Q$.
By the inequalities $1 - 1/x \le \log x \le x - 1$,
\[
- \frac{k |A_3|}{c-1} \le \log \left( 1 - \mathbb E [B_{3,i}] \right) \le - \frac{k |A_3|}{c}
\]
Also,
\[
\mathbb E [B_{1,i}] \cdot \mathbb E[B_{2,i}] = \left(1 - \left( 1 - \frac 1c \right)^{k |A_1|} \right) \left(1 - \left( 1 - \frac 1c \right)^{k |A_2|} \right)
\]
so by the inequality,
\[
\log (1 - (1 - (1 - x)^a) (1 - (1 - x)^b)) \ge - ab x^2, \quad a,b > 0, x \in [0,1];
\]
we have that
\[
- k^2 \frac{|A_1| |A_2|}{c^2} \le \log \left( 1 - \mathbb E [B_{1,i}] \cdot \mathbb E[B_{2,i}] \right) \le 0.
\]
Furthermore, notice that the LHS is minimized when $|A_1| = |A_2| = |N(x) \triangle N(y)| / 2$,
\[
- k^2 \frac{|A_1| |A_2|}{c^2} \ge - k^2 \frac{|N(x) \triangle N(y)|^2}{4c^2}.
\]
We then have that 
\begin{align*}
&\log(c - \mathbb E Q) = \log c + \log \left( 1 - \mathbb E [B_{1,i}] \cdot\mathbb E [B_{2,i}] \right) + \log \left( 1 - \mathbb E [B_{3,i}] \right) \\
&\quad \le \log c - \frac{k |N(x) \cap N(y)|}{c}
\end{align*}
and
\[
\log(c - \mathbb E Q) \ge \log c - k^2 \frac{|N(x) \triangle N(y)|^2}{4c^2} - \frac{k |N(x) \cap N(y)|}{c-1}.
\]

\begin{algorithm}[H]
\caption{A Standard Bloom Filter}
\label{alg:bloom-filter}
  \begin{compactitem}[leftmargin=*]
  \item[] {\bf class} {\tt BloomFilter:}
    \begin{compactitem}[leftmargin=*]
      \item[] {\bf def} $\mathtt{constructor}(\texttt{self}, c, \cbr{\mathtt{h_t}(\cdot): t = 1,\ldots,k})$:
        \begin{compactitem}[leftmargin=*]
          \item[] $\mathtt{self.b}[i] = 0\quad \forall i=1,\ldots,c$
          \item[] $\mathtt{self.h_{t} = h_{t}}\quad \forall i = 1,\ldots,k$
        \end{compactitem}
      \item[]
      \item[] {\bf def} $\mathtt{add(self, x)}$:
        \begin{compactitem}[leftmargin=*]
          \item[] $\mathtt{self.b}[\mathtt{self.h_t(x)}] = 1\quad\forall t = 1,\ldots,k$
        \end{compactitem}
      \item[]
      \item[] {\bf def} $\mathtt{union(self, bf)}$:
        \begin{compactitem}[leftmargin=*]
          \item[] $\mathtt{self.b}[i] \leftarrow \mathtt{self.b}[i] \mid \mathtt{bf}.\mathtt{b}[i]\quad \forall i=1,\ldots,c$
        \end{compactitem}
      \item[]
      \item[] {\bf def} $\mathtt{size(self)}$:
        \begin{compactitem}[leftmargin=*]
        \item[] {\bf return} $\ceil{-\frac{c}{k} \log\rbr{1 - \frac{\mathtt{nnz(self.b)}}{c}}}$
        \end{compactitem}
    \end{compactitem}
  \end{compactitem}
\end{algorithm}

\subsection{Co-Factorization with Graph Information}

Co-Factorization of Rating and Graph Information (Co-Factor) \cite{singh2008relational, liang2016factorization} is ideologically very different from GRMF and GRWMF, because it does not use graph information as regularization term. Instead it treats the graph adjacency matrix as another rating matrix, sharing one-sided latent factors with the original rating matrix. Co-Factor minimizes the following objective function:
\begin{align}
\min_{U,V} \sum_{(i,j) \in \Omega_R} \left(R_{i,j} - u_i^\top v_j \right)^2 &+ \frac{\lambda}{2} (\|U\|_F^2 + \|V\|_F^2 + \|V'\|_F^2) + \sum_{(i,j) \in \Omega_G} \left(G_{i,j} - u_i^\top v'_j \right)^2,
\label{eq:co-factor}\end{align}
where $U \in \dR^{n \times r}, V \in \dR^{m \times r}, V' \in \dR^{n \times r}$.
We can extend Co-Factor to incorporate our DNA-d 
by replacing $G$ with $B$ in \eqref{eq:co-factor}, where $B\in \dR^{n \times c}$ is the Bloom filter bipartite graph adjacency matrix of $n$ real-user nodes and $c$ pseudo-user nodes, similar to $B$ as in \eqref{eq:G1}. We call the extension Co-Factor\_DNA-$d$.


\subsection{Graph Convolutional Matrix Completion}
Graph Convolutional Matrix Completion (GC-MC) is a graph convolutional network (GCN) based geometric matrix completion method \cite{berg2017graph}. In \cite{berg2017graph}, the side information graph is represented as the adjacency matrix of the side graph and these one-hot encodings are treated as features for nodes in the graph. Convolutions of these features are performed on the bipartite rating graph. We find in our experiments that using these one-hot encodings of the graph as feature is an inferior choice both in terms of performance and speed. To capture higher order side graph information, it is better to use $G + \alpha G^2$ for some constant $\alpha$ and this alternate choice usually gives smaller generalization error than the original GC-MC method. However, it is hard to explicitly calculate $G + \alpha G^2$ and store the entire matrix for a large graph for the same reason described in Section~\ref{sec:grmf}. Again, we can use graph DNA to efficiently encode and store the higher order information before feeding it into GC-MC. We show in our experiments that this outperforms current state-of-the-art GCN methods \cite{berg2017graph, monti2017geometric}. 


\begin{figure*}
\begin{tabular}{cc}
\hspace{-8pt}
\includegraphics[width=0.4\linewidth]{speed_douban_v1.png} &
\includegraphics[width=0.4\linewidth]{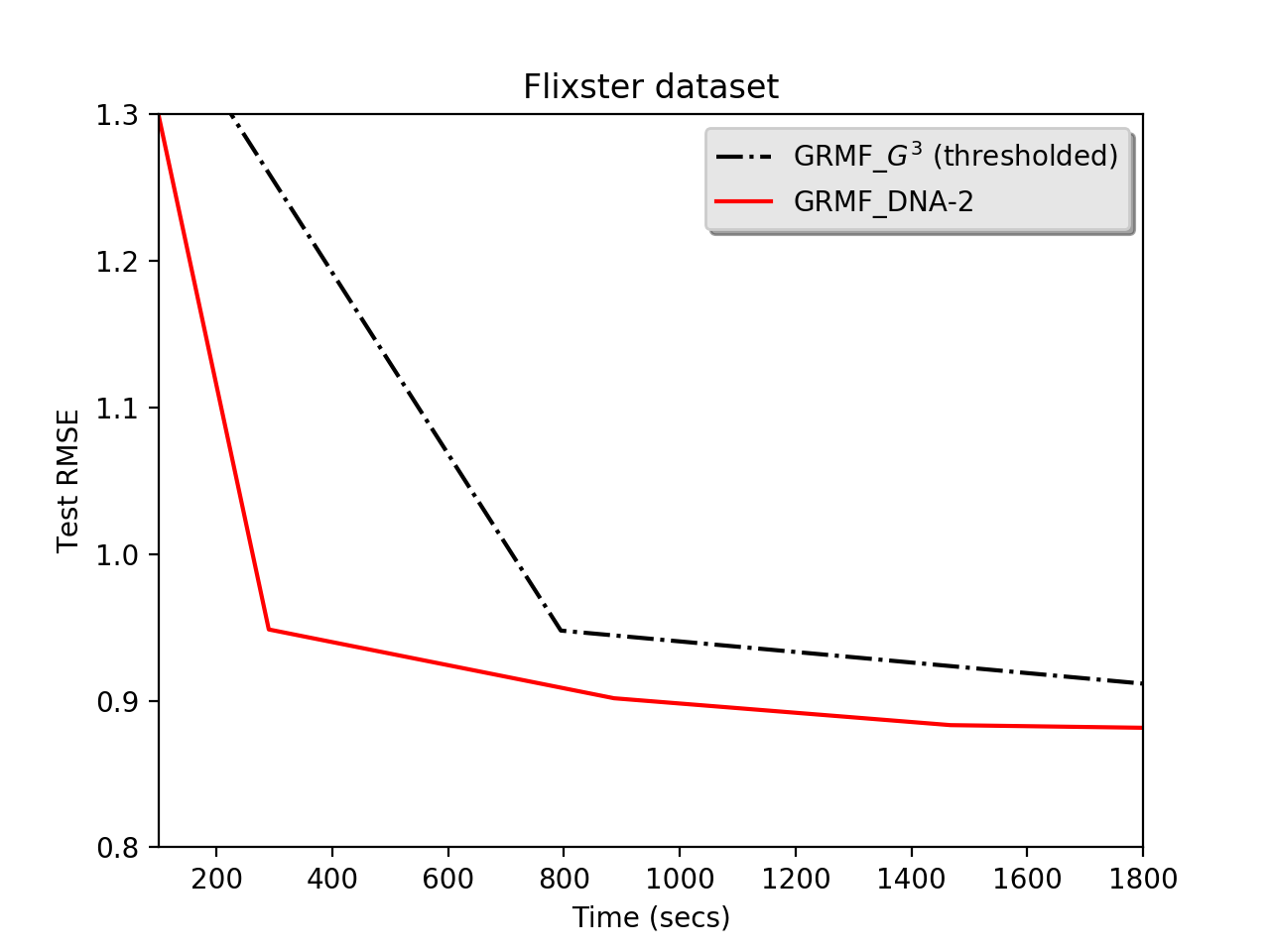} 
\end{tabular}
\caption{Compare Training Speed of GRMF, with and without Graph DNA. }
\label{fig:speed2}
\vspace{-10pt}\end{figure*}

\subsection{Simulation Study}
In the simulation we carried out, we set the number of users $n = 10,000$ and the number of items $m = 2,000$. We uniformly sample $5\%$ for training and $2\%$ for testing out of the total $nm$ ratings. We choose $T = 3$ so the graph contains at most $6$-hop information among $n$ users. We use rank $r = 50$ for both user and item embeddings. We set influence weight $w = 0.6$, i.e. in each propagation step, $60\%$ of one user's preference is decided by its friends (i.e. neighbors in the friendship graph). We set $p = 0.001$, which is the probability for each of the possible edges being chosen in Erd$\tilde{o}$s-R$\acute{e}$nyi graph $G$. A small edge probability $p$, influence weight $w < 1.0$, and a not too-large $T$ is needed, because we don't want that all users become more or less the same after $T$ propagation steps.

\subsection{Metrics}\label{sec:metric}

We omit the definitions of RMSE, Precision@$k$, NDCG@$k$, MAP as those can be easily found online.
HLU: Half-Life Utility \cite{breese1998empirical, shani2008mining} is defined as:
        \begin{equation}
            \text{HLU} = \frac{1}{n}\sum_{i = 1}^n \text{HLU}_i,
        \end{equation} where $n$ is the number of users and $\text{HLU}_i$ is given by:
        \begin{equation}
            \text{HLU}_i = \sum_{l = 1}^{k} \frac{ max(R_{i\Pi_{il}} - d, 0)}{2^{(j - 1) / (\alpha - 1)}},
        \end{equation}
        where $R_{i\Pi_{il}}$ follows previous definition, $d$ is the neural vote (usually the rating average), and $\alpha$ is the viewing halflife. The halflife is the number of the item on the list such that there is a $50$-$50$ chance the user will review that item~\cite{breese1998empirical}.

\begin{algorithm}
\caption{Simulation of Synthesis Data}
\label{alg:sim}
\begin{algorithmic}[1]
\Require $n$ users, $m$ items, rank $r$, influence weight $w$, $T$ propagation steps
\Ensure $R_{\text{tr}} \in \dR^{n \times m}$, $R_{\text{te}} \in\dR^{n \times m}$, $G \in\dR^{n \times n}$
\State Randomly initialize $U \in \dR^{n \times r}, V \in \dR^{m \times r}$ from standard normal distribution
\State Generate a random undirected Erd$\tilde{o}$s-R$\acute{e}$nyi graph $G$ with each edge being chosen with probability $p$
\For{$t = 1, ..., T$}
  \For{$i = 1, ..., n$}
      \State $\tilde{U}_i = w \cdot \sum_{j: (i, j) \in G}U_j + (1 - w) \cdot U_i$
  \EndFor
  \State Set $U = \tilde{U}$
\EndFor
\State Generate rating matrix $R = U V^T$
\State Random sample observed user/item indices in training and test data: $\Omega_{\text{tr}}, \Omega_{\text{te}}$
\State Obtain $R_{\text{tr}} = \Omega_{\text{tr}} \circ R, R_{\text{te}} = \Omega_{\text{te}} \circ R$
\State \textbf{return} rating matrices $R_{\text{tr}}, R_{\text{te}}$, user graph $G$
\end{algorithmic}
\end{algorithm}

\begin{table}
  \caption{Compare Bloom filters of different depths and sizes an on Synthesis Dataset. Note that the number of bits of Bloom filter is decided by Bloom filter's maximum capacity and tolerable error rate (i.e. false positive error, we use $0.2$ as default).}
  \vskip -0.1in
  \label{tab:sim_bits}
  \resizebox{0.8\textwidth}{!}{
  \begin{tabular}{lcccccccc}
    \toprule
    methods & max capacity & $c$ bits & nnz ratio & RMSE ($ \times 10^{-3}$) & \% Relative Graph Gain  \\
    \midrule
    GRMF\_$G^2$ & -  & - & -   & 2.6543  & 59.5903\\
    GRMF\_DNA-1 & 20 & 135 & 0.217  & 2.4303  & 163.8734 \\
    GRMF\_DNA-1 & 50 & 336 & 0.093  &  2.4795 &  140.9683 \\
    GRMF\_DNA-2 & 20 & 135 & 0.880 & 2.4921  & 135.1024 \\
    GRMF\_DNA-2 & 50 & 336  & 0.608  & 2.4937  & 134.3575 \\
    GRMF\_DNA-2 & 100 & 672  & 0.381  & 2.4510  & 154.2365 \\
    GRMF\_DNA-2 & 200 & 1,341 & 0.215  & 2.4541  & 152.7933  \\
    GRMF\_DNA-3 & 200 & 1,341 & 0.874  & 2.4667  & 146.9274 \\
    GRMF\_DNA-3 & 600 & 4,020 & 0.525  & 2.4572  & 151.3500 \\
    GRMF\_DNA-3 & 1,000 & 6,702 & 0.364  & 2.4392  & 159.7299 \\
    GRMF\_DNA-3 & \bfseries{1,500} & \bfseries{10,050}& \bfseries{0.262}  & \bfseries{2.4247}  & \bfseries{166.4804} \\
    GRMF\_DNA-4 & 2,000 & 13,401& 0.743 & 2.5532 &  106.6573 \\
    GRMF\_DNA-4 & 4,000 & 26,799& 0.499  & 2.4466 & 156.2849 \\
  \bottomrule
\end{tabular}
}
\end{table}

\begin{table*}
    \caption{Compare nnz of different methods on Douban and Flixster datasets. GRMF\_$G^4$ and GRMF\_DNA-2 are using the same $4$-hop information in the graph but in different ways. Note that we do not exclude potential overlapping among columns.}
    \vskip -0.1in
    \label{tab:nnz}
     \resizebox{1.0\textwidth}{!}{
    \begin{tabular}{lllccccccr}
    \toprule
    Dataset & methods & $R_{\text{tr}}$ & $G$ & $G^2$ & $G^3$ & $G^4$ & $B$ & total nnz \\
    \midrule
    \multirow{8}{*}{Douban}
    &MF   & 9,803,098 &  -  & -   & -     & - & -    & 9,803,098  \\
    &GRMF\_$G$   & 9,803,098 & 1,711,780 & -   & -  & -   & -     & 11,514,878  \\
    &GRMF\_$G^2$  & 9,803,098 & 1,711,780 & 106,767,776   & -  & -   & -     & 118,282,654  \\
    &GRMF\_$G^3$ & 9,803,098 & 1,711,780 & 106,767,776   & 2,313,572,544     & -   & -  & 2,431,855,198  \\
    &GRMF\_$G^4$ & \bfseries{9,803,098} & \bfseries{1,711,780} & \bfseries{106,767,776}   & \bfseries{2,313,572,544}     & \bfseries{8,720,553,105}    & - &  \bfseries{11,152,408,303} \\
    &GRMF\_DNA-1  & 9,803,098 &  0  & -   & -   & -  &    8,834,740  &  18,637,838 \\
    &GRMF\_DNA-2  & 9,803,098 &  1,711,780  & -   & -   & -  &   142,897,900   &  154,412,778 \\
    &GRMF\_DNA-3  & 9,803,098 &  1,711,780  & -   & -  & -   &   928,159,604   &  939,674,482 \\
    \midrule
    \multirow{8}{*}{Flixster}
    &MF   & 3,619,304 &  -  & -   & -     & -  & -   &  3,619,304 \\
    &GRMF\_$G$   & 3,619,304 & 2,538,746 & -   & -     & -  & -   &  6,158,050 \\
    &GRMF\_$G^2$  & 3,619,304 &2,538,746 & 130,303,379   & -     & -  & -   & 136,461,429  \\
    &GRMF\_$G^3$ & 3,619,304 & 2,538,746 &130,303,379   & 2,793,542,551     & -  & -   & 3,060,307,359   \\
    &GRMF\_$G^4$ & \bfseries{3,619,304} & \bfseries{2,538,746} & \bfseries{130,303,379}   & \bfseries{2,793,542,551}      & \bfseries{12,691,844,513}   & - &  \bfseries{15,752,151,872} \\
    &GRMF\_DNA-1  & 3,619,304 &  0  & -   & -   & -  &  12,664,952    &  16,284,256 \\
    &GRMF\_DNA-2  & 3,619,304 &  2,538,746  & -   & -   & -  &   181,892,883   &  188,050,933 \\
    &GRMF\_DNA-3   & 3,619,304 &  2,538,746  & -   & -  & -   &   1,185,535,529   & 1,191,693,579  \\
    \bottomrule
\end{tabular}
}
\vskip -0.15in
\end{table*}

\begin{table}
  \caption{Compare GRMF Methods of different ranks for Explicit Feedback on Flixster Dataset.}
  \label{tab:rank}
  \vskip -0.15in
  \resizebox{0.65\textwidth}{!}{
  \begin{tabular}{llccc}
    \toprule
    Rank & methods & test RMSE ($ \times 10^{-1}$)  & \% gain   \\
    \midrule
    \multirow{2}{*}{10}
    &GRMF\_$G^2$                    & 8.7849    & -     \\
    &GRMF\_DNA-3  & \bfseries{8.7383}  & \bfseries{0.8262}\\
    \midrule
    \multirow{2}{*}{20}
    &GRMF\_$G^2$                    & 8.9179   & -     \\
    &GRMF\_DNA-3  & \bfseries{8.7565}  & \bfseries{1.8098}\\
    \midrule
    \multirow{2}{*}{30}
    &GRMF\_$G^2$                   & 9.0865    & -     \\
    &GRMF\_DNA-3  & \bfseries{8.9255}    & \bfseries{1.7719}\\
  \bottomrule
\end{tabular}
}
\vskip -0.15in
\end{table}

\subsection{Graph Regularized Weighted Matrix Factorization for Implicit feedback}

We use the rank $r = 10$, negatives' weight $\rho = 0.01$ and measure the prediction performance with metrics MAP, HLU, Precision@$k$ and NDCG$@k$ (see definitions of metrics in Appendix~\ref{sec:metric}).

We follow the similar procedure to what is done before in GRMF and co-factor: we run all combinations of tuning parameters of $\lambda_l \in \{0.01, 0.1, 1, 10, 100\}$ and $\lambda_g \in \{0.01, 0.1, 1, 10, 100\}$ for each method on validation data for fixed number $40$ epochs and choose the best combination as the parameters to use on test data. We then report the best prediction results during first $40$ epochs on test data with the chosen parameter combination.

\subsection{Explore effects of rank}\label{rank}
Next we investigate whether the proposed DNA coding can achieve consistent improvements when varying the rank in the GRMF algorithm.
In Table~\ref{tab:rank}, we compare the proposed GRMF\_DNA-3 with GRMF\_$G^2$, which achieves the best RMSE without using DNA coding in the previous tables.
The results clearly show that the improvement of the proposed DNA coding is consistent over different ranks and works even better when rank is larger.

\subsection{Reproducibility}
\label{sec:reproduce}
To reproduce results reported in the paper, one need to download data (douban and flixster) and third-party C++ Matrix Factorization library from the link \url{https://www.csie.ntu.edu.tw/~cjlin/papers/ocmf-side/}. One can simply follow README there to compile the codes in Matlab and run one-class matrix factorization library in different modes (both explicit feedback and implicit feedback works). The advantage of using this library is that the codes support multi-threading and runs quite fast with very efficient memory space allocations. It also supports with graph or other side information. All three methods' baseline can be simply run with the tuning parameters we reported in the Table~\ref{tab:repro-2}, ~\ref{tab:repro-3}, ~\ref{tab:repro-4} in Appendix.

To reproduce results of our DNA methods, one need to generate Bloom filter matrix $B$ following Algorithm~\ref{alg:graph-bloom}. 
We will provide our python codes implementing Algorithm~\ref{alg:graph-bloom} and Matlab codes converting into the formats the library requires.

For baselines and our DNA methods, We perform a parameter sweep for $\lambda_l \in \{0.01, 0.1, 1, 10, 100\}$ and $\lambda_g \in \{0.01, 0.1, 1, 10, 100\}$ as well as for $\alpha \in \{0.0001, 0.001, 0.01, 0.1, 0.3, 0.7, 1\}$, for $\beta \in \{0.005, 0.01, 0.03, 0.05, 0.1 \}$ when needed. We run all combinations of tuning parameters for each method on validation set for $40$ epochs and choose the best combination as the parameters to use on test data. We then report the best test RMSE in first $40$ epochs on test data with the chosen parameter combination. 
We provide all the chosen combinations of tuning parameters that achieves reported optimal results in results tables in the Table~\ref{tab:repro-2}, ~\ref{tab:repro-3}, ~\ref{tab:repro-4} in Appendix. 
One just need to exactly follow our procedures in Section~\ref{sec:app} to construct new $\dot{G}, \dot{U}$ to replace the $G, U$ in baseline methods before feeding into Matlab.

As to simulation study, we will also provide python codes to repeat our Algorithm~\ref{alg:sim} to generate synthesis dataset. One can easily simulate the data before converting into Matlab data format and running the codes as before. The optimal parameters can be found in Table~\ref{tab:repro-1}. For all the methods, we select the best parameters $\lambda_l$ and $\lambda_g$ from $\{0.01, 0.1, 1, 10, 100\}$. For method GRMF\_$G^2$, we  tune an additional parameter $\alpha\in \{0.0001, 0.001, 0.01, 0.1, 0.3, 0.7, 1\}$.
For the thrid-order method 
GRMF\_$G^3$, we tune $\beta\in \{0.005, 0.01, 0.03, 0.05, 0.1 \}$ in addition to $\lambda_l, \lambda_G, \alpha$.
Due to the speed constraint, we are not able to tune a broader range of choices for $\alpha$ and $\beta$ as it is too time-consuming to do so especially for douban and flixster datasets. For example, it takes takes about 3 weeks using 16-cores CPU to tune both $\alpha, \beta$ on flixster dataset. We run each method with every possible parameter combination for fixed $80$ epochs on the same training data, tune the best parameter combination based on a small predefined validation data and report the best RMSE results on test data with the best tuning parameters during the first $80$ epochs. Note that only on the small synthesis dataset, we calculate full $G^3$ and report the results. On real datasets, there is no way to calculate full $G^4$ to utilize the complete $4$-hop information, because one can easily spot in Table~\ref{tab:nnz} the number of non-zero elements (nnz) is growing exponentially when the hop increases by $1$, which makes it impossible for one to utilize complete $3$-hop and $4$-hop information.

In Table~\ref{tab:repro-2}, one can compare magnitude of optimal $\alpha$ and $\beta$ to have a good idea of whether $G$ or $G^2$ is more useful. $G$ represents shallow graph information and $G^2$ represents deep graph information. If one already run GRMF\_$G^2$, one can then use this as a preliminary test to decide whether to go deep with
DNA-3 ($d = 3$) to capture deep graph information or simply go ahead with DNA-1 ($d = 1$) to fully utilize shallow information. 
For douban dataset, we have $\alpha = 0.05 >  0.0005 = \beta$, which implies shallow information is important and we should fully utilize it. It explains why DNA-1 is performing well both in terms of performance and speed on douban dataset. It is worth noting that 
GRMF\_DNA-1's Bloom filter matrix $B$ contains much more nnz than that of $G$ in Table~\ref{tab:nnz} though $20\%$ less than that of $G^2$. On the other hand, for flixster dataset, we have $\alpha = 0.01 <  0.1 = \beta$, which implies in this dataset deeper information is more important and we should go deeper. That explains why here 
GRMF\_DNA-3 ($6$-hop) achieves about $10$ times more gain than using $1$-hop 
GRMF\_$G$.  

\begin{table*}
  \caption{Compare Matrix Factorization for Explicit Feedback on Synthesis Dataset. The synthesis dataset has $10,000$ users and $2,000$ items with user friendship graph of size $10,000 \times 10,000$. Note that the graph only contains at most $6$-hop valid information. GRMF\_$G^6$ means GRMF with $G + \alpha \cdot G^2 + \beta \cdot G^3 + \gamma \cdot G^4 + \epsilon \cdot G^5 + \omega \cdot G^6.$ GRMF\_DNA-$d$ means depth $d$ is used.}
  \label{tab:repro-1}
   \resizebox{0.8\textwidth}{!}{
  \begin{tabular}{lccccccccc}
    \toprule
    methods & test RMSE ($ \times 10^{-3}$) & $\lambda_l$ & $\lambda_g $ & $\alpha$ & $\beta$ & $\gamma$ & $\epsilon$ &$\omega$ &\% gain over baseline \\
    \midrule
    MF                  & 2.9971 & 0.01  & -    & -     & - & - & -  & - & -  \\
    GRMF\_$G$               & 2.7823 & 0.01  & 0.01 & -  & -  & - & -  & - & 7.16693\\
    GRMF\_$G^2$  & 2.6543 & 0.01  & 0.01 & 0.3   & - & - & - &-  & 11.43772\\
    GRMF\_$G^3$ & 2.5687 & 0.01 & 0.01 & 0.01 & 0.05 & - &- &  -  &14.29382\\
    GRMF\_$G^4$ & 2.5562 & 0.01 & 0.01 & 0.01 & 0.05 & 0.1& - &-  & 14.71088\\
    GRMF\_$G^5$  & 2.4853 & 0.01 & 0.01 & 0.01 & 0.05 & 0.1 & 0.1&-  & 17.07651\\
    GRMF\_$G^6$  & 2.4852 & 0.01 & 0.01 & 0.01 & 0.05 & 0.1 & 0.1& 0.01 &17.07984\\
    GRMF\_DNA-1   & 2.4303 & 0.01 & 0.01  & - & - & - & - & -  &18.91161 \\
    GRMF\_DNA-2  & 2.4510 & 0.01 & 0.01  & - & - & - & - & -  &18.22095 \\
    GRMF\_DNA-3  & \bfseries{2.4247} & \bfseries{0.01} & \bfseries{0.01}  & - & - & - &- & -  &\bfseries{19.09846}\\
    GRMF\_DNA-4  & 2.4466 & 0.01 & 0.01  & - & - & - & - & -  &18.36776\\

  \bottomrule
\end{tabular}
}
\end{table*}

\begin{table*}
  \caption{Compare Matrix Factorization methods for Explicit Feedback on Douban and Flixster data. We use rank $r = 10$.} 
  \label{tab:repro-2}
   \resizebox{0.8\textwidth}{!}{
  \begin{tabular}{lllcccccc}
    \toprule
    Dataset & methods & test RMSE ($ \times 10^{-1}$) & $\lambda_l$ & $\lambda_g $ & $\alpha$ & $\beta$ & \% gain over baseline \\
    \midrule
    \multirow{8}{*}{Douban}
    &MF              & 7.3107 & 1   & -   & -     & -  & -     \\
    &GRMF\_$G$                 & 7.2398 & 0.1 & 100 & -     & -  & 0.9698\\
    &GRMF\_$G^2$ & 7.2381 & 0.1 & 100 & 0.001 & -  & 0.9930\\
    &GRMF\_$G^3$ (full)  & 7.2432  & 0.1 & 100 & 0.05 & 0.0005 & 0.9350\\
    &GRMF\_$G^3$ (thresholded) & 7.2382 & 0.1 & 100 & 0.05 & 0.0005 & 0.9917\\
    &GRMF\_DNA-1    & 7.2191 & 0.1 & 100 & - & - & 1.2689 \\
    &GRMF\_DNA-2   & 7.2359 & 1 & 10 & - & - & 1.0232 \\
    &GRMF\_DNA-3  & \bfseries{7.2095} & \bfseries{0.01} & \bfseries{100} & - & - & \bfseries{1.3843}\\
    \midrule
    \multirow{8}{*}{Flixster}
    &MF                & 8.8111 & 0.1 & 1   & -    & -     & -      \\
    &GRMF\_$G$               & 8.8049 & 0.01 & 1    & -     & -  & 0.0704\\
    &GRMF\_$G^2$ & 8.7849 & 0.01 & 1  & 0.05 & -  & 0.2974\\
    &GRMF\_$G^3$  (full) & 8.7932 & 0.1 & 1 & 0.01 & 0.1 & 0.2032\\
    &GRMF\_$G^3$  (thresholded) & 8.7920 & 0.01 & 1 & 0.01 & 0.1 & 0.2168\\
    &GRMF\_DNA-1    & 8.8013  & 0.01 & 1   & - & - & 0.1112  \\
    &GRMF\_DNA-2   & 8.8007 & 0.1 & 1  & - & - & 0.1180 \\
    &GRMF\_DNA-3 & \bfseries{8.7453} & \bfseries{0.1} & \bfseries{100} & - & - & \bfseries{0.7468}\\
  \bottomrule
\end{tabular}
}
\end{table*}

\begin{table*}
  \caption{Compare Co-factor Methods for Explicit Feedback on Douban and Flixster Datasets. We use rank $r = 10$ for both methods.}
 \label{tab:repro-3}
  \resizebox{0.6\textwidth}{!}{
  \begin{tabular}{llcccc}
    \toprule
    Dataset & methods & test RMSE ($ \times 10^{-1}$) & $\lambda_l$  & \% gain over baseline  \\
    \midrule
    \multirow{2}{*}{Douban}
    & co-factor\_$G$                     & 7.2743 & 1   & -     \\
    & co-factor\_DNA-$3$ & \bfseries{7.2674} & \bfseries{1}   & \bfseries{0.5923}\\
    \midrule
    \multirow{2}{*}{Flixster}
    & co-factor\_$G$                  & 8.7957 & 0.01   & -     \\
    & co-factor\_DNA-$3$ & \bfseries{8.7354} & \bfseries{0.01}   & \bfseries{0.8591}\\
  \bottomrule
\end{tabular}
}
\end{table*}

\begin{table*}
  \caption{Compare Weighted Matrix Factorization with Graph for Implicit Feedback on Douban and Flixster Datasets. We use rank $r = 10$ for both methods and all metric results are in $\%$.}
  \vskip -0.1in
  \label{tab:repro-4}
  \resizebox{0.7\textwidth}{!}{
  \begin{tabular}{llccccccccc}
    \toprule
    Dataset & Methods & MAP & HLU & P@$1$ & P@$5$ & NDCG@$1$ & NDCG@$5$ & $\lambda_l$ & $\lambda_g $ \\
    \midrule
    \multirow{2}{*}{Douban}
    &WMF\_$G$               & 8.340 & 13.033 & 14.944 & 10.371 & 14.944 & 12.564 & 0.01 & 10    \\
    &WMF\_DNA-3 & \bfseries{8.400} & \bfseries{13.110} & \bfseries{14.991} & \bfseries{10.397} & \bfseries{14.991} & \bfseries{12.619} & \bfseries{1}    & \bfseries{1}     \\
     \midrule
    \multirow{2}{*}{Flixster}
    &WMF\_$G$              & 10.889 & 14.909 & 12.303 & 7.9927 & 12.303 & 12.734 & 10 & 0.1    \\
     &WMF\_DNA-3 & \bfseries{11.612} & \bfseries{15.687} & \bfseries{12.644} & \bfseries{8.1583} & \bfseries{12.644} & \bfseries{13.399} & \bfseries{1}    & \bfseries{1}     \\

  \bottomrule
\end{tabular}
}
\vskip -0.15in
\end{table*}

\subsection{Code}\label{sec:code}
We will make our code available on Github in the final version of the paper

\end{document}